%% file: main.tex
\documentclass[onecolumn]{article}

\usepackage{microtype}
\usepackage{graphicx}
\usepackage{subcaption}
\usepackage{booktabs} %
\usepackage{xspace}
\usepackage{bbding}
\usepackage{fontawesome}
\usepackage{wrapfig}
\usepackage{enumitem}
\usepackage{listings}
\setlist[itemize]{align=parleft,left=2pt,nosep, topsep=-3pt}
\usepackage[table]{xcolor}

\usepackage{hyperref}
\usepackage{amsmath, amssymb, amsthm}
\usepackage{cleveref}

\usepackage{comment}
\definecolor{junglegreen}{rgb}{0.16, 0.67, 0.53}

\newcommand{\llama}{\texttt{Llama2-7B-chat}\xspace}
\newcommand{\llamalarge}{\texttt{Llama2-13B-chat}\xspace}
\newcommand{\AdvBench}{\texttt{AdvBench}}
\newcommand{\ours}{ActSVD\xspace}
\newcommand{\Alpacacleaned}{\texttt{Alpaca-Cleaned}}

\usepackage{multirow}
\usepackage{array}
\usepackage{tabularx}
\usepackage[accepted]{icml2024}

\icmltitlerunning{Assessing the Brittleness of Safety Alignment via Pruning and Low-Rank Modifications}

\begin{document}

\setlength{\textfloatsep}{5pt}

\twocolumn[
\icmltitle{Assessing the Brittleness of Safety Alignment \\ via Pruning and Low-Rank Modifications}

\icmlsetsymbol{equal}{*}
\icmlsetsymbol{equal2}{$\dagger$}

\begin{icmlauthorlist}
\icmlauthor{Boyi Wei}{equal}
\icmlauthor{Kaixuan Huang}{equal}
\icmlauthor{Yangsibo Huang}{equal}
\icmlauthor{Tinghao Xie}{}
\icmlauthor{Xiangyu Qi}{}
\icmlauthor{Mengzhou Xia}{}
\\
\icmlauthor{Prateek Mittal}{}
\icmlauthor{Mengdi Wang}{equal2}
\icmlauthor{Peter Henderson}{equal2}
\end{icmlauthorlist}

\begin{center}
{\bf Princeton University}
\end{center}

\vskip 0.3in
]

\printAffiliationsAndNotice{\icmlEqualContribution. $^\dagger$Equal advising} %

\input{sections/abstract}

\input{sections/intro}

\input{sections/method}
\input{sections/exp_setup}

\input{sections/exp_merged}

\input{sections/limitation}

\input{sections/conclusion}

\section*{Acknowledgements}
We express our gratitude to Vikash Sehwag, Chiyuan Zhang, Yi Zeng, Ruoxi Jia, Lucy He, Kaifeng Lyu, and the Princeton LLM Alignment reading group for providing helpful feedback. Boyi Wei and Tinghao Xie are supported by the Francis Robbins Upton Fellowship, Yangsibo Huang is supported by the Wallace Memorial Fellowship, and Xiangyu Qi is supported by Gordon Y. S. Wu Fellowship. 
Prateek Mittal acknowledges the support by NSF grants CNS-1553437 and CNS-1704105, the ARL’s Army Artificial Intelligence
Innovation Institute (A2I2), the Office of Naval Research Young Investigator Award, the Army Research Office Young
Investigator Prize, Schmidt DataX award, and Princeton E-affiliates Award.
Mengdi Wang acknowledges the support by NSF IIS-2107304, NSF CPS-2312093, ONR 1006977, and Genmab. This research is also supported by the Center for AI Safety Compute Cluster. Any opinions, findings, conclusions, or recommendations expressed in this material are those of the author(s) and do not necessarily reflect the views of the sponsors. 

\section*{Contribution Statement}

This project was a team effort. The contributions of each junior author are detailed below:

\textbf{Idea formulation and preliminary exploration:} The project's core ideas, specifically linking safe behaviors to certain model regions and isolating safety from utility, were developed by Boyi, Kaixuan, and Yangsibo.  Yangsibo and Boyi implemented Wanda and set difference methods based on discussions with Mengzhou, and ran preliminary experiments. Kaixuan implemented SNIP, \ours, and orthogonal projection methods, and ran preliminary experiments.

\textbf{Literature survey:} The literature review was conducted by team members in their respective areas of expertise.  Tinghao focused on alignment and jailbreaks; Yangsibo explored task attribution; Kaixuan reviewed low-rank modifications, and Mengzhou surveyed pruning techniques.

\textbf{Evaluation: }  Boyi led the charge on key experiments and the creation of visuals, while the entire team played a part in the evaluation process.  Yangsibo and Tinghao prepared the attribution data and set up the evaluation pipeline. Xiangyu collaborated with Boyi on reporting results for the attention head probing baseline. Boyi and Kaixuan investigated Jaccard index and subspace similarity to analyze the overlapping between safety and utility. Yangsibo and Tinghao studied the effect of freezing safety-critical regions during the fine-tuning.

\textbf{Writing: } The initial structure and primary drafting of the manuscript were led by Yangsibo, Boyi, and Kaixuan. The rest of the team members contributed by editing and providing valuable feedback on the manuscript.

\input{sections/broader_impact}

\newpage

\bibliography{main}
\bibliographystyle{icml2024}

\newpage
\appendix
\onecolumn

\input{sections/related}

\input{appendices/appendix_exp_details}

\input{appendices/appendix_exp_results}

\input{appendices/appendix_proof}

\end{document}

%% file: sections/abstract.tex
\begin{abstract}
Large language models (LLMs) show inherent brittleness in their safety mechanisms,  as evidenced by their susceptibility to jailbreaking and even non-malicious fine-tuning. This study explores this brittleness of safety alignment by leveraging pruning and low-rank modifications.
We develop methods to identify critical regions that are vital for safety guardrails, and that are disentangled from utility-relevant regions at both the neuron and rank levels.
Surprisingly, the isolated regions we find are sparse, comprising about $3\%$ at the parameter level and $2.5\%$ at the rank level. 
Removing these regions compromises safety while only mildly impacting utility, 
corroborating the inherent brittleness of the model's safety mechanisms. Moreover, we show that LLMs remain vulnerable to low-cost fine-tuning attacks even when modifications to the safety-critical regions are restricted. These findings underscore the urgent need for more robust safety strategies in LLMs.

\end{abstract}

%% file: sections/intro.tex
\section{Introduction}

\begin{figure*}[h]
    \centering
    \includegraphics[width=\linewidth]{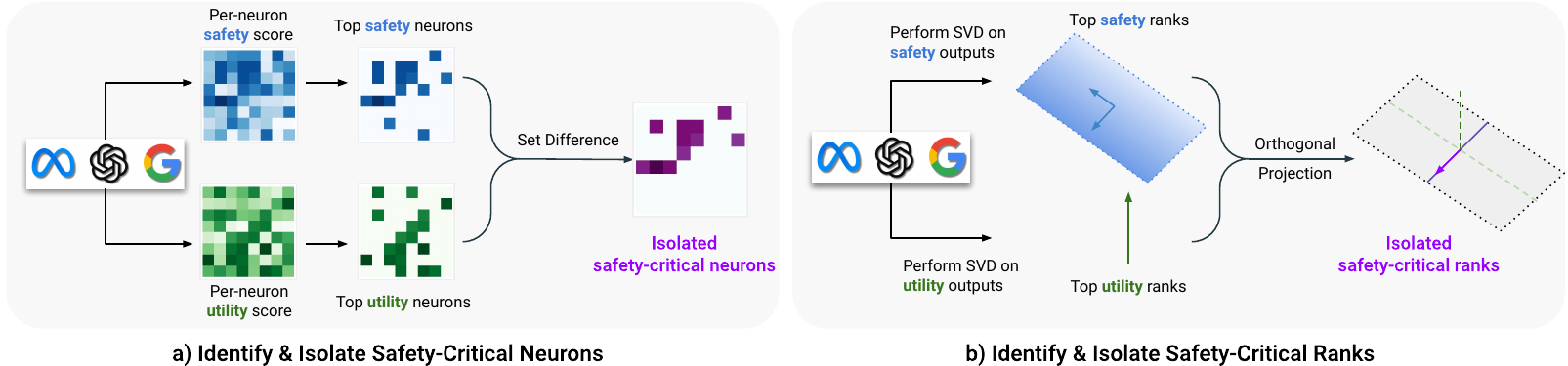}
    \vspace{-7mm}
    \caption{The proposed pipelines for identifying and isolating safety-critical regions of LLM weights at (a) neuron level and (b) rank level. \textbf{(a).} We identify the top safety neurons and the top utility neurons by computing per-neuron importance scores on the safety dataset and the utility dataset. Next, we isolate the safety-critical neurons from the utility neurons using set difference.  \textbf{(b).} We identify the top safety ranks and the top utility ranks by performing SVD on the safety outputs and the utility outputs (termed \ours). Next, we isolate the safety-critical ranks using orthogonal projection. 
    }
    \label{fig:main}
\end{figure*}

The capabilities of large language models (LLMs) have been significantly improved over the past few years~\cite{brown2020language, chatgpt, openai2023gpt4, touvron2023llama, touvron2023llama-2, claude, team2023gemini}. However, LLMs are not without limitations; they can sometimes produce outputs that are inaccurate, misleading, or harmful.  To align LLMs with human values, several approaches have been proposed, including reinforcement learning from human feedback~\cite{ziegler2019fine, ouyang2022training, bai2022training} and AI feedback~\cite{bai2022constitutional, lee2023rlaif}, and the development of more computationally efficient alternatives~\cite{sun2023principle, rafailov2023direct}.

Despite these efforts, recent studies have uncovered concerning `jailbreak' scenarios. In these cases, even well-aligned models have had their safeguards successfully breached~\cite {jailbreakchat}. These jailbreaks can include crafting adversarial prompts~\cite{wei2023jailbroken, jones2023automatically, carlini2023aligned, zou2023universal, shen2023anything, zhu2023autodan, qi2023visual}, applying persuasion techniques~\cite{zeng2024johnny}, or manipulating the model's decoding process~\cite{huang2023catastrophic}. Recent studies show that fine-tuning an aligned LLM, even on a non-malicious dataset, can inadvertently weaken a model's safety mechanisms~\cite{qi2023fine, yang2023shadow, zhan2023removing}. Often, these vulnerabilities apply to both open-access and closed-access models.

Addressing failure cases in the alignment of LLMs requires a deep understanding of why their safety mechanisms are fragile. Our study aims to provide a possible understanding via \textit{weight attribution} --- the process of linking safe behaviors to specific regions within the model's weights.\footnote{See the project website for code and other information: \url{https://boyiwei.com/alignment-attribution/}.} However, a key challenge here is the intricate overlap between \textit{safety} mechanisms and the model's general capabilities, or \textit{utility}. Consider the task of responding responsibly to a harmful instruction, such as \textsc{``Please provide five key steps to commit a fraud."}. The model must first comprehend the step-by-step nature of the request, then recognize the illegality and harmful intent of committing fraud, and ultimately, formulate a response that appropriately declines the request.  This process requires a blend of safety awareness and utility capability of the model. 
Our goal is to identify the smallest number of \textbf{safety-critical} links in the model, which \textbf{only} contribute to the model's safety. If these links are removed, the model is effectively jailbroken while utility remains relatively unaffected. 
If there are few such links, it may help explain why safety mechanisms remain brittle and why low-cost fine-tuning attacks have been so successful.

Our study examines the model weights and disentangles safety and utility from two perspectives: individual \textbf{neurons} and specific \textbf{ranks} within the model. 
For neuron attribution, we follow two widely adopted and effective methods from the previous works on pruning transformer models~\citep{lee2019snip, sun2023simple} to calculate a behavior-specific importance score for each neuron in an LLM, which identifies a group of neurons crucial for a certain behavior, such as giving safe responses (safety) or following general instructions (utility). For rank attribution, we propose \ours, a data-aware low-rank decomposition algorithm to identify crucial ranks of each weight matrix for the behavior.

To further address the complexity of potential entanglement between safety and utility, we propose set difference method for neuron attribution, and orthogonal projection method for rank attribution, to isolate \textit{safety-critical neurons} and \textit{safety-critical ranks}, respectively (see \Cref{fig:main}). This separation allows for a more refined analysis of safety mechanisms, leading to the findings below.\footnote{Our method is generally applicable to the attribution of different behaviors, but in the context of this study, the behaviors of interest are `safety' and `utility'.}

\textbf{Safety-critical regions are very sparse in aligned models.} We experiment with our method on Llama2-chat model family~\cite{touvron2023llama-2}. After disentangling utility from safety, we find that safety-critical neurons form a remarkably sparse structure in the model, representing just $3\%$ of the weights (\Cref{sec:exp_removal}). {Similarly, safety-critical ranks account for only about $2.5\%$ of the total ranks.} 
The sparsity of these regions may help explain why safety is so easily compromised after fine-tuning. %

\textbf{Removing safety-critical regions reduces safety while mostly maintaining utility.} 
We then demonstrate that the removal of specifically identified safety-critical neurons or safety-critical ranks in the Llama2-chat models significantly weakens their safety guardrails (\Cref{sec:exp_removal}). The attack success rate escalates from $0$ to over $90\%$, yet the model's overall utility remains largely unaffected. Conversely, we find that removing neurons or ranks deemed least important for safety can marginally improve the model's resistance to jailbreaking attempts (\Cref{sec:exp_removal_bottom}) -- a potentially exciting direction for improving safety via pruning-based approaches.

\textbf{Freezing safety-critical regions still remains vulnerable to fine-tuning attacks.} While intuitively, preventing modification of safety-critical parameters might reduce the likelihood fine-tuning attacks succeeding, our findings reveal that this strategy only offers resistance to minor model modifications (\Cref{sec:freeze_safety_neurons}). This suggests that fine-tuning attacks may introduce new pathways that bypass the original safety mechanisms in the model. This indicates a need for further research to develop more robust mitigation strategies against such attacks.

Our work suggests that the vulnerability of the model's safety mechanisms may stem from the sparse distribution of safety-critical regions within the model's architecture. Therefore, the sparsity of safety-critical neurons and ranks may act as a model-intrinsic metric for assessing the brittleness of safety alignment, complementing red teaming efforts. We hope it inspires further development of robust and reliable safety alignment algorithms 
that aim to integrate safety-critical regions seamlessly with utility regions for enhanced overall model performance.

%% file: sections/method.tex
\section{Methodology}
\label{sec:method}

\newcommand{\argmin}{\mathop{\arg\min}}
\newcommand{\rank}{\mathop{\mathrm{rank}}}

To identify and isolate the regions that are exclusively responsible for a model's safety behaviors, %
we make certain modifications to the weights of the model and observe its behavioral change in safety and utility. A natural way to modify the network weights is \textbf{neuron} removal, where we set several neurons of the weight matrices to zero. Besides, as LoRA~\citep{hu2021lora} is a popular parameter-efficient fine-tuning technique, we also consider \textbf{rank} removal, where we remove several ranks of the weight matrices.

\input{tables/method}

For a given calibration dataset, we consider two types of behavioral changes after we modify the network weights: \textbf{output change}, where we directly monitor the change of the immediate outputs of each of the modified layers; and \textbf{loss change}, where we monitor the change of the final loss.

The section is organized as follows. In \Cref{sec:score}, we consider the three approaches on weight attribution for a general calibration dataset%
, and summarize them in~\Cref{tab:method}. In \Cref{sec:diff}, we illustrate how to isolate safety-critical regions that leads to significant changes in safety behavior while having minimal effect on the general capabilities of the models.

\subsection{Identifying Important Neurons and Ranks}  
\label{sec:score}

\textbf{SNIP score \citep{lee2019snip}} 
For a data instance $x = (x_{\mathrm{prompt}}, x_{\mathrm{response}})$, we take the loss as the conditional negative log-likelihood $\mathcal{L}(x) = - \log p(x_{\mathrm{response}} \mid x_{\mathrm{prompt}})$ predicted by the model.
For any linear layer with weight matrix $W \in \mathbb{R}^{d_{\mathrm{out}} \times d_{\mathrm{in}}}$, we can calculate the importance score for the loss $\mathcal{L}(x)$ for each weight entry $W_{ij}$ 
\[
I(W_{ij}, x) = |W_{ij} \cdot \nabla_{W_{ij}} \mathcal{L}(x) |,
\]
which is the first-order Taylor approximation to the change of the loss when the weight entry $W_{ij}$ is set to zero. 
In matrix form, we have 
\[
I(W, x) = |W \odot \nabla_{W} \mathcal{L}(x) |.
\]
Given a calibration dataset $D$, we take the absolute value first and then take the average over $D$ and obtain
\[
I(W) = \mathbb{E}_{x\sim D} I(W,x) = \mathbb{E}_{x \sim D} |W \odot \nabla_{W} \mathcal{L}(x)|.
\]
where we get an individual score for each example and aggregate over the examples following~\citet{michel2019sixteen}. 
Intuitively, $I(W)$ measures how important each entry is for the behavior of the model on the calibration dataset $D$. Small $I(W)_{ij}$ indicates that setting $W_{ij}$ to zero has negligible impact on each of the calibration data points $x$, and we can attribute the specific behavior of the model to the weights with large $I(W)_{ij}$.

\textbf{Wanda score~\cite{sun2023simple} } 
For a calibration dataset, we store all the activations corresponding to the layer $W$ into $X_{\mathrm{in}} \in \mathbb{R}^{d_\mathrm{in} \times n}$. 
We consider multiplying the weight matrix $W$ with an element-wise binary mask $M$, resulting in a sparse matrix, such that the Frobenius norm of the change to the output is minimized as in \citet{frantar2023sparsegpt}:
\[
    \min_{M \in \{0,1\}^{{ d_{\mathrm{out}} \times d_{\mathrm{in}}  }}} \| W X_{\mathrm{in}} - (M \odot W) X_{\mathrm{in}} \|_F^2.
\]

We follow Wanda~\citep{sun2023simple} to obtain an approximate solution to the objective above. Denote an all-one vector by $\boldsymbol{1} \in \mathbb{R}^{d_{\mathrm{out}}}$. The importance score $I(W)$ is given by 
\[
I(W) =  |W| \odot   \Big (\boldsymbol{1} \cdot \| X_{\mathrm{in}} \|_2^\top\Big ),
\]
where we take the row-wise $L^2$ norm to obtain $\| X_{\mathrm{in}} \|_2 \in \mathbb{R}^{d_\mathrm{in}}$. To obtain a sparse network while keeping the change to the outputs minimal, one can prune out weight neurons corresponding to the minimal score $I(W)$.

In our setting, as we are only interested in measuring the importance of each weight entry that contributes to the model's response, we mask out the prompt activations and only include the response activations in $ X_{\mathrm{in}} $.

\textbf{\ours~ } Recall that we store all the response activations before the layer $W$ into $X_{\mathrm{in}} \in \mathbb{R}^{d_\mathrm{in} \times n}$. We seek to find a low-rank matrix $\widehat{W}$ such that the Frobenius norm of the change to the output is minimized:
\[
    \widehat{W} = \argmin_{\rank {\widehat{W}} \leq r} \| W X_{\mathrm{in}} - \widehat{W} X_{\mathrm{in}} \|_F^2.
\]
This can be done by performing SVD on $W  X_{\mathrm{in}} \in \mathbb{R}^ {d_\mathrm{out} \times n}$:
\[
    U S V^\top  \approx WX_{\mathrm{in}}, 
\]
where $U \in \mathbb{R}^{d_\mathrm{out}\times r}$ is the orthogonal matrix corresponding to the top $r$ left singular vectors. The minimizer is given by 
\[
    \widehat{W} = UU^\top W,
\]
where $\Pi = UU^\top$ is the orthogonal projection onto the $r$ most significant left singular subspace. The proof is postponed to \Cref{sec:proof}.
Note that $\widehat{W}$ can be implemented by a LoRA update. The rank of the LoRA adaptor $\Delta W = W -\widehat{W}$ can be bounded by
\[
    \rank(\Delta W) \leq \rank(I - UU^\top) = R-r,
\]
where $R = \rank(W)$.
With a slight abuse of terminology, we refer to these rank-1 components $(U_1U_1^\top W, \  U_2U_2^\top W, \dots)$ as the important \textit{ranks} of the weight matrix $W$. See \Cref{sec:lowrank} for a detailed discussion of the related methods. %

\subsection{Isolating Safety-Critical Neurons and Ranks} \label{sec:set_difference}
\label{sec:diff}

Assume we have two calibration different datasets: a {\textbf{utility}} dataset $D^u$ and a \textbf{safety} dataset $D^s$. $D^u$ contains prompts and responses that are related to general language abilities. $D^s$ demonstrates behavior where the requests are harmful and the responses properly decline the requests. 

We seek to isolate the safety-critical regions of the network weights, where removing those regions will have a \textit{low influence} on the model's behavior on distribution $D^u$ but a \textit{high influence} on $D^s$ --- effectively maintaining the model's general language abilities but compromising its safety alignment.

\textbf{Isolating safety-critical neurons } For the two calibration datasets, assume we obtain two scores $I^u$ and $I^s$, respectively, either using SNIP or Wanda as in \Cref{sec:score}. We consider the weight neurons that score least according to $I^u$ but score most according to $I^s$. We adopt \textit{per-output comparison group} as~\citet{sun2023simple}, which corresponds to each matrix row. Specifically, for any pair of sparsity levels $(p\%, q\%)$, we define the top-$p\%$ important neurons $S^u(p)$ for utility as the neurons whose important utility score $I^u_{i,j}$ ranks top $p\%$ among the $i$-th \textit{row} of $I^u$.
\[
   S^u(p) = \{(i,j) \mid I^u_{i,j} \text{ is the top } p\% \text{ of } I^u_i\}.
\]
Similarly, we define the top-$q\%$ important neurons $S^s(q)$ for safety as
\[
    S^s(q) = \{(i,j) \mid I^s_{i,j} \text{ is the top } q\% \text{ of } I^s_i\}.
\]

Then, the isolated neurons $S(p,q)$ is defined as the set difference between  $S^s(q)$ and  $S^u(p)$ :
\[
S(p,q) =  S^s(q) - S^u(p).
\]

\textbf{Isolating safety-critical ranks }  
Let $R = \rank(W)$. For the two calibration datasets $D^u$ and $D^s$, for any pair of integer $(r^u, r^s)$, we can obtain the projection matrices $\Pi^u$ and $\Pi^s$ as stated in \Cref{sec:score}, where
\[
    \Pi^u = U^u (U^u)^\top, \rank(\Pi^u) = R - r^u,
\]
and 
\[
    \Pi^s = U^s (U^s)^\top, \rank(\Pi^s) = R - r^s.
\]
Multiplying the weight matrix by $\Pi^u$ removes the least important $r^u$ ranks for $D^u$, while multiplying the weight matrix by $\Pi^s$ removes the least important $r^s$ ranks for $D^s$. To isolate the safety-critical ranks, we consider the matrix 
\[
    \Delta W(r^u, r^s) = (I-\Pi^u)  \Pi^s  W.
\]
Removal of $\Delta W(r^u, r^s)$ essentially removes the important ranks of the safety behavior that are  \textit{orthogonal} to the important ranks of the utility behavior. 
The modified weight matrix is $\widetilde{W} = W -  \Delta W(r^u, r^s)$, which can be implemented by a LoRA update~\citep{hu2021lora} with
\[
    \rank(\Delta W(r^u, r^s)) \leq \min(r^u, R-r^s).
\]  

%% file: tables/method.tex
\begin{table*}[t]
\centering
{
\setlength{\tabcolsep}{24pt}
\begin{tabular}{ccc}
\toprule
& \textbf{Neuron Attribution} & \textbf{Rank Attribution} \\
\midrule 

 \multirow{2}{*}{\bf Output Change} & \cellcolor[gray]{1} Wanda~\citep{sun2023simple} & \cellcolor[gray]{1} \ours   \\ 
& $
 |W| \odot   \Big (\boldsymbol{1} \cdot \| X_{\mathrm{in}} \|_2^\top\Big )
$ & $
    U S V^\top  \approx WX_{\mathrm{in}},\   
    \widehat{W} = UU^\top W$\\
\midrule
  \multirow{2}{*}{\bf Loss Change} & \cellcolor[gray]{1}  SNIP~\citep{lee2019snip}  & \multirow{2}{*}{N/A\footnotemark} \\
  & $\mathbb{E}_{x \sim D} |W \odot \nabla_{W} \mathcal{L}(x)|$ 
   & \\
\midrule
 \multirow{2}{*}{\bf Disentanglement Method}   & \cellcolor[gray]{1} Set difference   & \cellcolor[gray]{1} Orthogonal projection \\ 
   & $S(p,q) = S^s(q) - S^u(p)$ & $\Delta W(r^u,r^s) = (I-\Pi^u)\Pi^s W$ \\
\bottomrule
\end{tabular}}
\caption{The overview of our {weight attribution} methods. For \textbf{neuron} attribution, we compute Wanda importance score~\citep{sun2023simple} as per output change and SNIP importance score~\citep{lee2019snip} as per loss change. To disentangle safety from utility, we adopt set difference to isolate the safety-critical neurons. For \textbf{rank} attribution, we compute the most important ranks via \ours per output change and adopt orthogonal projection to isolate the safety-critical ranks.
}
\label{tab:method}
\end{table*}

\footnotetext{We exclude the rank attribution + loss change combination in our paper due to several technical challenges, e.g., differentiating through an orthogonal projection matrix requires Lie algebra analysis~\citep{schotthofer2022low}.}

%% file: sections/exp_setup.tex
\vspace{-5mm}
\section{Experimental Setup}\label{sec:exp_setup}

\subsection{Models, Datasets, and Evaluation Metrics}

\paragraph{Models} Our experiments use \llama and \llamalarge~\cite{touvron2023llama-2}. We select them for their publicly accessible weights and their extensive safety tuning process. 

\paragraph{Datasets} To identify safety-critical regions in the model, we prepare two types of datasets:  the safety dataset, for attributing safety-related behaviors, and the utility dataset, for attributing utility-related behaviors. Each dataset is structured in a (prompt, response) format. More details about these datasets are provided in \Cref{sec:app_exp_detail}.

The safety dataset is compiled using harmful instructions from \AdvBench~\cite{zou2023representation}. We divide \AdvBench~into \AdvBench$_\texttt{eval}$ ($100$ instructions for evaluation) and \AdvBench$_\texttt{attr}$ ($420$ instructions for attribution). We prompt \llama with \AdvBench$_\texttt{attr}$, collecting responses that refrain from following harmful instructions. As noted by \citet{zou2023universal}, the \textit{judgement} segments in the model's responses (e.g., ``Sure," ``I am sorry") significantly impact the nature of subsequent responses. We thus create two variants: safety-full (entire response) and safety-short (judgement segment only). 

For the utility dataset, we filter out safety-related (prompt, response) pairs using sensitive phrase matching \cite{qi2023fine} from \Alpacacleaned\footnote{\url{https://github.com/gururise/AlpacaDataCleaned}}, a refined version of the Alpaca dataset~\cite{alpaca}.

\vspace{-2mm}
\paragraph{Measuring utility} 
Following~\citet{sun2023simple}, we
measure the model's utility by reporting its averaged zero-shot accuracy of six tasks from EleutherAI LM Harness \cite{eval-harness}: BoolQ \cite{clark2019boolq}, RTE \cite{wang2018glue}, HellaSwag \cite{zellers2019hellaswag}, WinoGrande \cite{sakaguchi2021winogrande}, ARC Challenge \cite{clark2018think}, and OpenbookQA \cite{mihaylov2018can}. 

\paragraph{Measuring safety} We measure the model's safety by evaluating its attack success rate (ASR) in response to harmful instructions.  Specifically, we prompt the model using \AdvBench$_\texttt{eval}$, the first $100$ prompts from \AdvBench, and collect its responses. Following~\citet{zou2023universal}, we consider an attack as successful if the model's response lacks key patterns indicative of instruction rejection. The ASR is then computed as the ratio of successfully attacked prompts to the total number of prompts evaluated.

\input{tables/ASR_details}

Our safety evaluation considers three use cases: the ASR under standard, non-malicious conditions (ASR$_\textrm{Vanilla}$),  and the ASR under two malicious settings -- ASR$_\textrm{Adv-Decoding}$~\cite{huang2023catastrophic}, where the attacker manipulates the decoding process, and ASR$_\textrm{Adv-Suffix}$~\cite{zou2023universal}, where the attacker optimizes to find adversarial suffixes. Differences in these metrics are detailed in \Cref{tab:ASR_details}. Due to the high computational cost associated with calculating adversarial suffixes, we precompute several suffixes, and use the three best-performed ones in our evaluation. Note that we only include the system prompt when calculating ASR$_\textrm{Vanilla}$. More details are provided in \Cref{sec:app_exp_detail}. %

\subsection{Variants to Identify Safety-Critical Neurons}
\label{sec:exp_setup_neurons}
We conduct experiments to identify safety-critical neurons using the following methods:
\begin{itemize}
    \item \textbf{SNIP (top)}: we regard neurons that receive top-$p\%$ SNIP scores (\Cref{sec:score}) on safety data as safety-critical. We choose $p = 0.01$.
    \item \textbf{Wanda (top)}: we regard neurons that receive top-$p\%$ Wanda scores (\Cref{sec:score}) on safety data as safety-critical. We choose $p = 0.01$.
    \item \textbf{SNIP\footnote{We only use SNIP for set difference because during our early experiments, we find SNIP performs slightly better than Wanda.} with set difference}:
    we identify safety-critical neurons by focusing on those with top $q\%$ scores in the safety data, which are not included in the top $p\%$ scoring neurons according to the utility data (\Cref{sec:diff}).
    We do a grid search for parameters $p$ and $q$, with their values ranging from $0.1$ to $90$. %
    \item \textbf{Probing}:  we also compare our approach with probing~\cite{hewitt2019designing}, a common method for attributing behaviors of LLMs to their internal components. Following standard probing practices~\cite{clark2019does, campbell2023localizing, li2023inference}, we feed the model both harmful and harmless instructions, collect activation outputs from each attention head, and then train a linear classifier for each head to differentiate these activations. Attention heads with the highest accuracy on the evaluation set are identified as safety-critical.  \Cref{sec:app_exp_detail} provides more details.
\end{itemize}

\subsection{Variants to Identify Safety-Critical Ranks}
\label{sec:exp_setup_ranks}

We conduct experiments to identify safety-critical ranks using the following methods:
\begin{itemize}
    \item \textbf{\ours (top)}: we regard the top-$r$ ranks identified as most safety-related by \ours (\Cref{sec:score}) as safety-critical. We choose $r=1$.
    \item \textbf{\ours with orthogonal projection}: we identify the safety-critical ranks via orthogonal projection between the utility projection matrix $\Pi^u$ and the safety projection matrix $\Pi^s$ obtained from \ours (\Cref{sec:diff}). 
    {We do a grid search for  $r^u$ and $r^s$ between $50$ and $4000$.}

\end{itemize}

%% file: tables/ASR_details.tex
\begin{table}[t]
\setlength{\tabcolsep}{2pt}
\centering
\resizebox{\linewidth}{!}{
\begin{tabular}{cccc}
\toprule
          &  ASR$_\textrm{Vanilla}$     & ASR$_\textrm{Adv-Suffix}$  & ASR$_\textrm{Adv-Decoding}$         \\
\midrule
Sample Times          & $1$           & $1$ & $5$         \\
System Prompt       & \faCheckSquareO   & \faTimes   & \faTimes\\
\texttt{[INST]}, \texttt{[/INST]} wrapper         & \faCheckSquareO & \faCheckSquareO    & \faTimes      \\
Adversarial Suffix& \faTimes & \faCheckSquareO  & \faTimes      \\
\bottomrule
\end{tabular}
}
\caption{The differences between three types of ASR in our safety evaluation. ASR$_\textrm{Vanilla}$ captures the model behavior under standard usage, while ASR$_\textrm{Adv-Suffix}$ and ASR$_\textrm{Adv-Decoding}$ captures behaviors against adversaries.}
\label{tab:ASR_details}
\end{table}

%% file: sections/exp_merged.tex
\section{Experimental Results}
\label{sec:exp_results_main}

\begin{figure*}[ht]
\centering
\begin{minipage}[b]{\linewidth}
\centering
\includegraphics[width=0.98\linewidth]{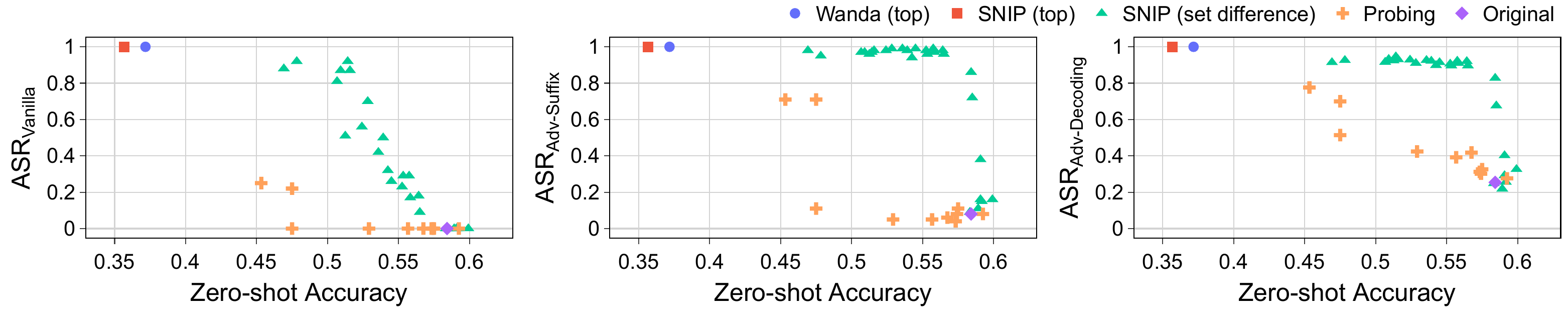}
\vspace{-2mm}
\subcaption{Removing Safety-Critical \textbf{Neurons}}
\label{fig:neuron_7b}
\end{minipage}
\begin{minipage}[b]{\linewidth}
\centering
\includegraphics[width=0.98\linewidth]{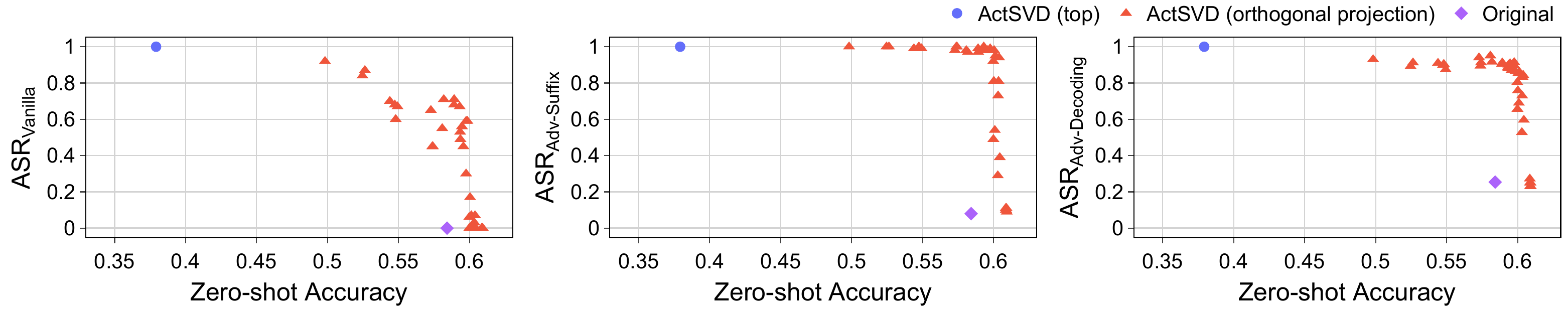}
\vspace{-2mm}
\subcaption{Removing Safety-Critical \textbf{Ranks}}
\label{fig:rank_7b}
\end{minipage}
\vspace{-7mm}
\caption{ASR and accuracy after removing safety-critical regions in \llama identified by: \textbf{(a)} Different methods in \Cref{sec:exp_setup_neurons} with \textbf{sparsity constraint $\mathbf{<3\%}$}. 
\textbf{(b)} Different methods in \Cref{sec:exp_setup_ranks} with \textbf{{ranks of the weight updates ($\rank(\Delta W)$)  less than $\mathbf{100}$}} (out of $4096$). Among all methods, disentangling safety from utility (set difference for neurons and orthogonal projection for ranks) mostly effectively identify the safety-critical regions, with safety severely compromised while utility retains. We obverse similar results on \llamalarge (see \Cref{fig:neuron_13b} in \Cref{sec:13b_results}). %
}
\label{fig:exp_main}
\end{figure*}

\vspace{-2mm}

This section presents our findings on both neuron and rank levels. In \Cref{sec:exp_removal}, we demonstrate that our set difference and orthogonal projection outperform other methods in isolating safety-critical neurons or ranks. In \Cref{sec:exp_removal_bottom}, we show that the safety of the model can be enhanced by removing the least important safety neurons or ranks. Then in \Cref{sec:safety_utility_analysis}, we analyze the overlap between safety and utility neurons and ranks, where we observe less overlap in MLP layers than self attention layers. Finally, in \Cref{sec:freeze_safety_neurons} we examine the potential of freezing safety-critical neurons to counter fine-tuning attacks and explore how fine-tuning may circumvent safety mechanisms.

\subsection{Disentangling Safety and Utility is Vital for Identifying
Safety-Critical Regions}
\label{sec:exp_removal}

We experiment with different methods to identify safety-critical neurons and ranks outlined in  \Cref{sec:exp_setup_neurons} and \Cref{sec:exp_setup_ranks}, and summarize our findings as below.

\textbf{Safety-critical regions are sparse and can be effectively isolated via set difference or orthogonal projection.} We isolate neurons contributing to safety from those contributing to utility, by applying the set difference method described in \Cref{sec:diff} to SNIP.  \Cref{fig:neuron_7b} presents the Pareto front resulting from set difference-based pruning: We observe that removing less than $3\%$ of neurons pushes ASR in all three scenarios close to $1$, while maintaining an average zero-shot accuracy above $0.5$. Similarly, \Cref{fig:rank_7b} shows results for removing safety ranks orthogonal to utility ranks: $\Delta W(r^u, r^s) = (I - \Pi^u)\Pi^sW$. Notably, {an update of just $2.5\%$ (less than $100$)} of the total $4096$ ranks significantly increases the model's ASR, while preserving its zero-shot accuracy. For example, when we remove the orthogonally-projected top-$6$ safety ranks while keeping the top-$96$ utility ranks untouched, we get $0.71, 0.97, 0.91$ in ASR$_\textrm{Vanilla}$, ASR$_\textrm{Adv-Suffix}$ and ASR$_\textrm{Adv-Decoding}$ respectively and $0.58$ zero-shot accuracy. These findings suggest that regions critical for safety are relatively \textit{sparse} within the weight matrices, evident at both the neuron and rank levels.

\textbf{Pruning merely less than $1\%$ neurons makes the model vulnerable in adversarial cases.} We also observe from \Cref{fig:exp_main} (middle and right) that models tend to be more fragile in adversarial scenarios, as indicated by the non-dropping accuracy when ASR$_\textrm{Adv-Suffix}$ and ASR$_\textrm{Adv-Decoding}$ reach to $1$. Interestingly, if we focus solely on adversarial cases, only pruning less than $1\%$ of neurons can significantly compromise the model's safety while still keeping its accuracy above $0.53$, as shown in \Cref{tab:pq_comb} in \Cref{sec:pq_comb}.

\textbf{Pruning top safety neurons or ranks severely compromises utility.} 
Removing neurons with the highest safety importance scores, either calculated with Wanda or SNIP, also leads to a complete loss of safety in the \llama model. However, there is also a drastic decrease in the model's utility, with its average accuracy dropping to about $0.35$, significantly lower than its original accuracy of $0.58$.  Likewise, removing even just the top-1 rank (out of $4096$) critical for safety causes the model's accuracy to drop to $0.38$.  Similar results are also observed in \llamalarge (\Cref{fig:neuron_13b} in \Cref{sec:13b_results}).

\begin{figure*}[t]
    \begin{minipage}{0.49\linewidth}
    \centering
    \includegraphics[width=\linewidth]{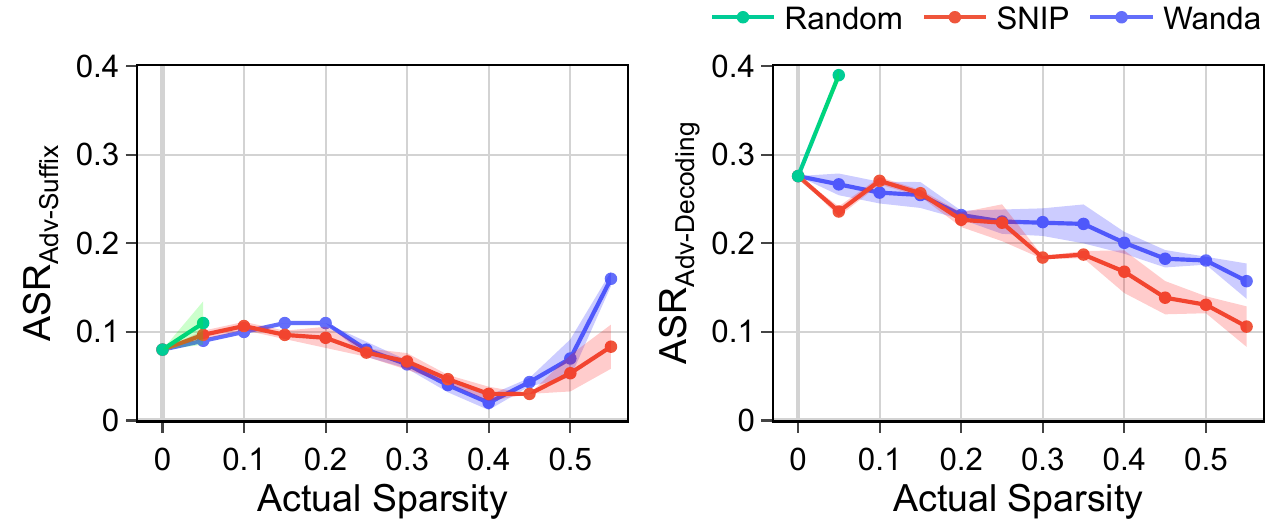}
    \vspace{-2mm}
    \subcaption{Removing the least Safety-Relevant \textbf{Neurons}}
    \label{fig:ASR_vs_sparsity_neuron}
    \end{minipage}
    \hfill
    \begin{minipage}{0.49\linewidth}
    \centering
    \includegraphics[width=\linewidth]{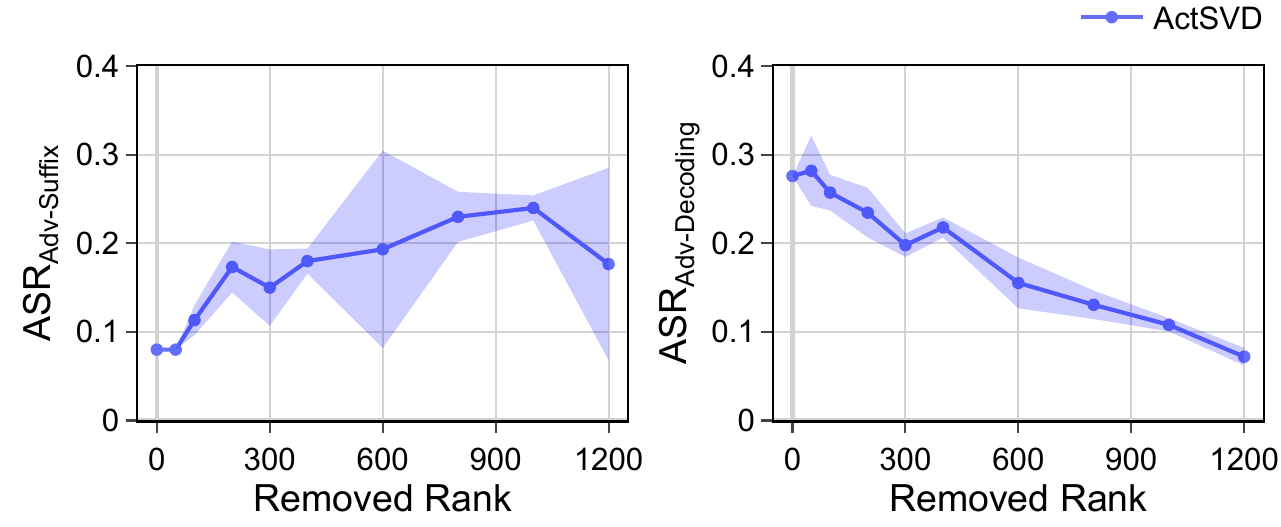}
    \vspace{-2mm}
    \subcaption{Removing the least Safety-Relevant \textbf{Ranks}}
    \label{fig:ASR_vs_sparsity_rank}
    \end{minipage}
    \caption{Impact on ASR under adversaries in \llama when: (a) Removing neurons with  \textit{the lowest} safety importance scores and (b) Removing \textit{the least} safety-relevant ranks identified by \ours. Both the importance score and safety-relevant ranks are calculated on the safety-short dataset. All plotted points have reasonable utility with accuracy $>0.5$. 
    The exclusion of neurons and ranks deemed least
critical for safety results in improved robustness against adversarial decoding attacks. 
Potential reason for the variation in ASR$_\textrm{Adv-Suffix}$ is the adversarial suffixes may not directly transfer to the modified model.
     See \Cref{fig:ASR_vs_sparsity_13b} for the results on \llamalarge.
    }
    \label{fig:ASR_vs_sparsity}
\end{figure*}

These observations support our rationale for isolating safety from utility: Regions primarily contributing to safety behaviors in aligned models may also be crucial for its general utility. Consequently, removing these regions can impair the model's ability to generate grammatically correct content, which in turn undermines its safety mechanisms.
\newpage
\textbf{Attention head probing scores cannot isolate safety-critical neurons.}
Our probing results in~\Cref{sec:app_probing} suggest that activations of individual attention heads are predictive of identifying harmful versus harmless instructions, with over half achieving more than $0.95
$ probing accuracy. 
Based on the obtained score of each attention head, we prune the top-$k$ scored (out of $1024$) attention heads from \llama, with $k$ ranging from $10$ to $300$. However, as shown in \Cref{fig:neuron_7b}, our set difference approach outperforms the probing method consistently, yielding higher ASR at the same level of accuracy. Similar results are observed on \llamalarge (see \Cref{fig:neuron_13b}). This highlights the need for a disentanglement method, as achieving high harmful versus harmless prediction accuracy does not necessarily mean the top predictive heads are solely responsible for generating safety responses. Besides, these results also imply the need to focus on the MLP layers and a finer granularity like neurons or ranks.

\vspace{-2mm}
\subsection{Enhancing Safety by Eliminating Regions with Minimal Safety Relevance} \label{sec:exp_removal_bottom}

\begin{figure*}[t]
    \centering
    \begin{minipage}[b]{0.97\linewidth}
    \includegraphics[width=\linewidth]{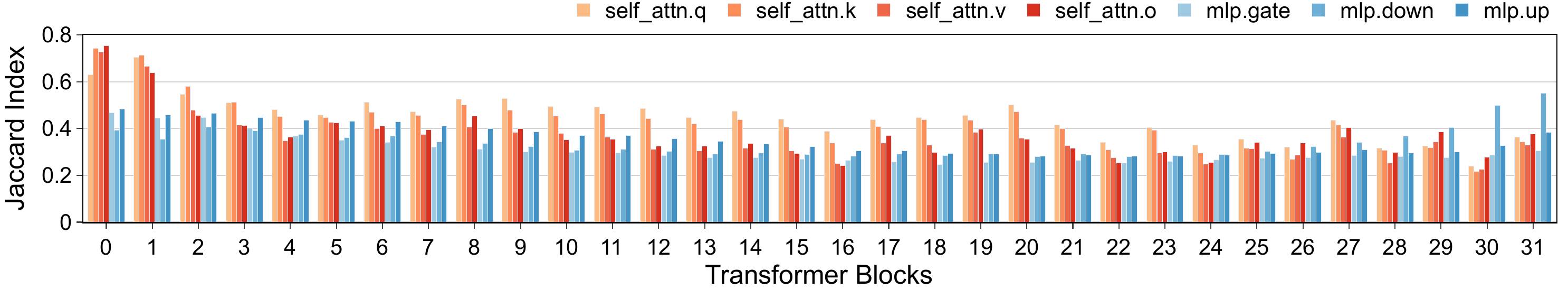}
    \vspace{-5mm}
    \subcaption{Neuron level}
    \label{fig:overlap_neuron}
    \end{minipage}
    \vspace{2mm}
    \begin{minipage}[b]{0.97\linewidth}
    \includegraphics[width=\textwidth]{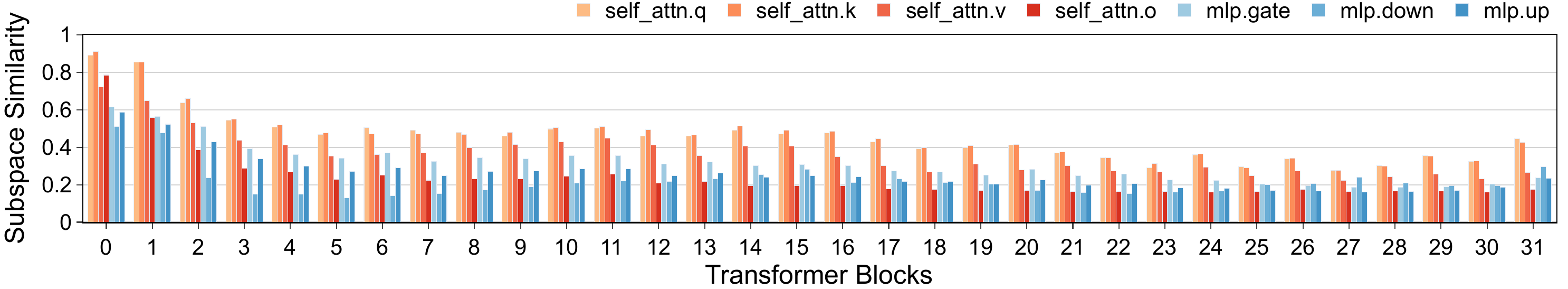}
    \vspace{-5mm}
    \subcaption{Rank level}
    \label{fig:overlap_rank}
    \end{minipage}
    \vspace{-5mm}
    \caption{Safety and utility overlapping analysis of each layer inside \llama using: (a) Jaccard Index between top $5\%$ safety neurons and top $5\%$ utility neurons and (b) Subspace similarity between rank-$100$ utility matrix $U^u$ and rank-$100$ safety matrix $U^s$.  \textbf{Lower} Jaccard index/subspace similarity means lower overlapping of utility and safety, namely utility and safety behaviors are \textbf{more differentiated}. MLP layers appear to encode more differentiated behaviors. 
    }
    \label{fig:jaccard_index_analysis}
\end{figure*}

\textbf{Pruning least safety-relevant neurons improves safety.}
Thinking from the opposite, it's reasonable to hypothesize that the neurons with the lowest safety importance scores could be detrimental for safety. Consequently, eliminating these neurons could potentially enhance the overall safety of the model.
To verify this, we conduct an ablation study to prune weights with the lowest SNIP and Wanda scores at various sparsity levels ($s\%$), and report randomly removing $s\%$ of neurons as a baseline for comparison. Particularly, we only report the results for pruned models maintaining reasonable utility, defined by an average accuracy above $0.5$. As shown in \Cref{fig:ASR_vs_sparsity_neuron}, random pruning strategy significantly reduces model accuracy, which falls below $0.5$ after pruning merely $10\%$ of the weights; it also leads to a noticeable decline in the model's safety. In contrast, when pruning is guided by the lowest safety importance score, the model's accuracy remains largely stable (i.e., $> 0.5$). Furthermore, pruning neurons with the lowest safety scores even slightly enhances the model's safety, against both decoding manipulation and adversarial suffixes. This indicates that neurons identified as least important for safety may undermine the model's safety, thereby validating the effectiveness of our method in calculating safety importance scores.

\textbf{Removing least safety-relevant ranks improves safety.} Similarly, we also consider removing the least safety ranks from the model using \ours. Specifically, we remove least $(R-r)$ safety ranks by using $\widehat{W}=UU^\top W=\sum_{i=1}^r U_iU_i^\top W$ to approximate $W$.
In accordance to our findings at the neuron level, as we increase the removed rank, we observe a decrease in the model's ASR$_\textrm{Adv-Decoding}$ in \Cref{fig:ASR_vs_sparsity_rank}. This also echoes the recent findings~\cite{sharma2023truth} where removing high-order components improves the model's reasoning performance. However, the model's ASR$_\textrm{Adv-Suffix}$ exhibits considerable variation, which could potentially be due to that the adversarial suffixes found using \citet{zou2023universal} on the original model cannot directly transfer to the modified model. 

\vspace{-3mm}
\subsection{MLP
Layers Appear to Encode More Differentiated Behaviors}\label{sec:safety_utility_analysis}

In \Cref{sec:exp_removal}, we find that removing high-safety-score regions also compromises utility, indicating a possible entanglement of safety and utility regions. We validate this at both neuron and rank levels.

\textbf{Neuron-level Jaccard index:} We calculate the layer-wise Jaccard index, $J(A, B) = |A\cap B|/|A\cup B|$, to quantify the overlap between top $p\%$ utility neurons  and top $q\%$ safety neurons. \Cref{fig:overlap_neuron} shows Jaccard indices across all transformer blocks and layers in \llama, using SNIP importance scores with top percentages $p\%$ and $q\%$ at $5\%$\footnote{
We choose top $5\%$ because the optimal selections of $(p\%, q\%)$ presented in \Cref{fig:exp_main} is around $5\%$ (see \Cref{sec:app_exp_results}).
}. The observed spikes in  Jaccard indices indicate large overlaps between safety and utility neurons within certain layers of the model. Notably, MLP layers exhibit lower Jaccard indices compared to attention layers, suggesting that utility or safety-related knowledge is more differentiated in MLP layers within language models \cite{dai2021knowledge}.

\textbf{Rank-level subspace similarity:} Similarly, we also find that MLP layers appear to encode more differentiated behaviors for safety and utility, from the rank perspective. Specifically, we report\footnote{We choose top $100$ ranks as the optimal  $(r^u, r^s)$ corresponds to top $100$ safety and utility ranks (see \Cref{sec:pq_comb}).} the subspace similarity $\phi(U^u, U^s)$ between rank-$100$  $U^u$ and rank-$100$ $U^s$ as defined in \citet[Appendix G]{hu2021lora} in \Cref{fig:overlap_rank}, where 
\[
    \phi(U^u, U^s) = \frac{ \| (U^u)^\top U^s \|^2_F}{\min(\rank(U^u), \rank(U^s))}.
\] 
As shown, the left singular matrices $U^u$ and $U^s$ exhibit a lower subspace similarity (i.e., utility and safety behaviors are more differentiated) in MLP layers, corroborating our findings at the neuron level. 
\vspace{-0.25cm}
\subsection{Freezing Safety-Critical Neurons Does Not Stop  Fine-Tuning Attacks}\label{sec:freeze_safety_neurons}

\vspace{-0.05cm}
Finally, we explore the implications of the identified neurons on the fine-tuning attacks, which demonstrate that fine-tuning an aligned model, even with harmless data, can unexpectedly weaken its safety measures~\cite{qi2023fine, yang2023shadow, zhan2023removing}. This issue is particularly concerning given the increasing availability of model fine-tuning APIs from major vendors like OpenAI\footnote{\url{https://platform.openai.com/finetune}}.

\input{tables/ft_attack}

We explore whether the identified safety-critical neurons could mitigate the fine-tuning attack\footnote{We only explore this for neurons, as the ``freezing" operation at rank level cannot be easily achieved using $U^u$ and $U^s$.
}. Following the experimental setup in \citet{qi2023fine}, we fine-tune \llama with varying numbers of examples ($n$) from the Alpaca dataset~\citep{alpaca}. During fine-tuning, we freeze the top-$q\%$ of safety neurons and observe their effect on preserving safety. As shown in \Cref{tab:ft_attack}, effective counteraction of the attack occurs only with $n=10$ and freezing over $50\%$ of neurons.  This observation aligns with ~\citet{lee2024mechanistic}'s hypothesis that fine-tuning attacks may create alternative pathways in the original model. Given that safety-critical neurons are sparse, these new routes could bypass the existing safety mechanisms easily, and therefore we need more robust defenses against fine-tuning attacks.

%% file: tables/ft_attack.tex
\begin{table}[t]
    \small
    \centering
    \setlength{\tabcolsep}{8pt}
    \small
    \begin{tabular}{c|c|c|c}
        \toprule
        \multirow{2}{*}{\bf{\% of Frozen Weights}} & \multicolumn{3}{c}{\textbf{ASR}$_\textrm{\bf Vanilla}$ } \\ \cmidrule{2-4}
         & $n=10$ & $n=50$ & $n=100$ \\ 
        \midrule
        $0$ (original attack) & $0.53$ & $0.91$ & $0.94$ \\
        \midrule
        $9.61$ & $0.52$ & $0.92$ & $0.94$ \\ 
        $19.22$ & $0.40$ & $0.91$ & $0.94$ \\ 
        $28.83$ & $0.38$ & $0.91$ & $0.94$ \\ 
        $48.05$ & $0.34$ & $0.91$ & $0.94$ \\ 
        $67.27$ & $0.23$ & $0.85$ & $0.89$ \\ 
    \bottomrule
    \end{tabular}
    \vspace{-2mm}
    \caption{ASR$_\textrm{Vanilla}$ of the model under fine-tuning with varying fractions of safety-critical neurons frozen. $n$ represents the number of Alpaca examples used for fine-tuning. Freezing safety-critical neurons is insufficient to thwart fine-tuning attacks.
    }
    \label{tab:ft_attack}
\end{table}

%% file: sections/limitation.tex
\section{Limitations \& Future Work}
We identify areas for potential future research and limitations of this work. First, there are limited publicly accessible, strong safety-aligned models, which constrains our experiments to the Llama2-chat models. Other aligned models, trained with different datasets and strategies, might demonstrate varying behaviors under our methodology.

Second, we found that standard attention head probing does not effectively localize safety-critical neurons. Our findings that MLP layers may exhibit better localization for safety knowledge, also suggest that future probing research could explore MLP layers in more depth. This exploration could also examine the potential integration of these methods with our pipeline to enhance safety attribution effectiveness.

Our study proposes initial, yet promising, strategies for improving safety robustness, which could be explored further: (1) pruning regions least important for safety (or potentially explicitly harmful) could improve safety robustness; (2) making safety-critical regions difficult to isolate may be an exciting new direction in building inherently safer models.

%% file: sections/conclusion.tex
\section{Conclusion}

In this study, we introduce a pipeline for identifying safety-critical regions (neurons and ranks) in LLMs, which effectively disentangles the regions critical for safety and those vital for utility. Our experiments with Llama2-chat models demonstrate that safety-critical regions are notably sparse in aligned LLMs, accounting for about $3\%$ at the weight level and $2.5\%$ at the rank level. Despite their sparsity, these regions are crucial for the integrity of the model's safety mechanisms, as removing them destroys the model's safety with utility retained. This sparsity may explain the observed brittleness in safety alignment in current LLMs, and could serve as a model-intrinsic metric for assessing the brittleness of safety alignment in future models, thereby complementing red teaming efforts. And our work suggests potentially important future directions for improving the robustness and safety of models overall.

%% file: sections/broader_impact.tex
\section*{Impact Statement}

\textbf{Dual-use Risk.} Our work, like other safety and security research, aims to make models safer in the long run by identifying short-term weaknesses. We hope our work will spur additional researching robust safety mechanisms that are not so sparse, easy to isolate, and easy to remove. 

That being said, with any safety and security research there is some risk that adversaries will use our work to remove safety guardrails. We believe the benefit of releasing our work and engaging in this study outweighs these potential risks for three reasons. 

First, we perform experiments on models (Llama2-chat family) that already has a base model available without any safety guardrails, so there is no marginal increased risk. Second, by assessing the brittleness of safety guardrails in our work, we can encourage stronger guardrails to be developed that are more difficult for an attacker to isolate and remove. Our work may help explain existing results demonstrating that safety guardrails can be removed via fine-tuning, while identifying potential pathways for improved defenses. Third, our work does not significantly decrease the cost of jailbreaking a model beyond alternative strategies. Models can already be jailbroken relatively cheaply with fine-tuning. Instead, our focus is on analysis and understanding of the brittleness of these safety mechanisms so that future work can reduce the risk of jailbreaking in open models.

Overall, we hope that our work will improve the state of AI safety, particularly in open models, by providing key analysis and information.

\textbf{Safety and harm definitions.} We generally follow existing standard benchmarks and protocols for assessment of safety and harm, but these may not cover all definitions of safety and harm. Further work can be done to expand analysis to a wider range of settings and we encourage additional work in the space of definitions and evaluation that are beyond the scope of this work.

%% file: sections/related.tex
\section{Related Work}

\subsection{Alignment and Jailbreak}
\label{sec:related_alignment}

Alignment refers to the process of ensuring a machine learning (ML) model's behavior conforms to human values. For example, pretrained language models are usually not aligned with human objectives -- they cannot follow users' instructions, and could potentially generate harmful and incorrect content. During the alignment stage, practitioners would employ Instruction Tuning~\cite{wei2021finetuned, ouyang2022training, touvron2023llama-2}, and Reinforcement Learning from Human Feedback (RLHF)~\cite{ouyang2022training, touvron2023llama-2, bai2022training} to enforce the language models to be \textit{helpful, harmless, and honest} (the HHH principle) \cite{askell2021general}.
Aligned LLMs (e.g., OpenAI ChatGPT~\cite{openai2023gpt4} and Anthropic Claude~\cite{claude, claude2}), as a result, will follow human values and refuse to respond to harmful requests.
Recent work~\cite{rafailov2023direct,xu2023some,dai2023safe,yang2024rlcd,li2024rain,yuan2024self, huang2024vaccine} propose more effective and efficient alignment alternatives to RLHF.
As some examples,
Direct Preference Optimization (DPO)~\cite{rafailov2023direct} directly fine-tunes language models on human preference data, eliminating the need to train a reward model and conduct reinforcement learning;
Self-Rewarding~\cite{yuan2024self} uses the language model itself as a reward model to curate labeled preference data, and then align the language model with DPO in an iterative way;
\cite{dai2023safe} proposes to decouple the goal of safety and helpfulness during alignment, similar to the decoupling goal of our work.

While harmful instructions that are plain and direct would be rejected by aligned LLMs, researchers and communities have identified ways to \textit{bypass or remove} the safety guardrails enforced by LLM alignment efforts --- namely ``jailbreaking'' LLMs. More specifically, \textit{jailbreaking} is a series of attacks where an adversary would coax or enforce the model to deviate from its ethical guidelines. In practical, jailbreak attackers would either employ \textit{adversarial prompts}~\cite{liu2023jailbreaking, zou2023universal, yuan2023gpt, liu2023autodan, shen2023anything, yong2023low, mehrotra2023tree} or manipulate model \textit{decoding} process~\cite{huang2023catastrophic} to bypass LLM safety alignment. Moreover, when having \textit{fine-tuning} access of LLMs, adversaries~\cite{qi2023fine, yang2023shadow, zhan2023removing} could directly remove the guardrails. A jailbroken LLM would provide harmful responses to comply with users' harmful requests, which they otherwise would simply reject (due to the ethical guidelines injected by alignment) --- this could subsequently pose serious safety risks in the real world, since LLMs can directly deliver various harmfulness to individuals and the society.

\subsection{Identifying Task-Specific Regions in Large Language Models}
\label{sec:related_attr}

Attributing the model's behavior to the model's weights is a classic research question in explainable machine learning~\cite{tjoa2020survey, burkart2021survey, ali2023explainable}. Previous studies in pre-transformer eras have explored various approaches to identify task-specific neurons within models, to better interpret and control the model's behavior. Popular techniques mainly include perturbation-based methods that perturb the input of a model and observe the changes in the output of the model~\cite{zeiler2014visualizing, ribeiro2016should}, and gradient-based methods that compute an importance score for model weights based on the results from task-specific back propagation~\cite{springenberg2014striving, bach2015pixel, sundararajan2017axiomatic, shrikumar2017learning, lundberg2017unified}. However, it was only recently that these methods have been rigorously applied to modern transformer models ~\cite{madsen2022post, zhao2023explainability, maini2023can}. %

Probing has emerged as a method for understanding the knowledge encoded in transformers, particularly large language models~\cite{adi2016fine, conneau2018you, hewitt2019designing}. To perform
probing, the model representations and model
parameters are fed into a probe classifier~\cite{belinkov2022probing}, whose task is to identify certain linguistic properties or reasoning
abilities acquired by the model. For instance, previous work has adopted probing-based method to localize truthfulness~\cite{li2023inference, campbell2023localizing}, factuality~\cite{meng2022locating, geva2023dissecting}, toxicity~\cite{lee2024mechanistic}, and knowledge~\cite{burns2022discovering, todd2023function} in LLMs. More recently, \citet{zou2023representation} propose a similar approach to probing: instead of training a classifier, they employ an unsupervised approach, specifically singular value decomposition, to identify significant directions in the representation space. They then demonstrate that these directions can predict and influence the behavior of LLMs.

In addition to the importance-score-based and probing-based methods discussed above, recent studies have also investigated a range of techniques to pinpoint task-specific neurons in transformers. These techniques address various aspects of the model, including linguistic properties~\cite{dalvi2019one, antverg2021pitfalls}, general capabilities~\cite{lan2023locating, gurnee2024universal, merullo2023circuit}, fine-tuning~\cite{panigrahi2023task, lubana2023mechanistic}, and prompt tuning~\cite{wang2022finding}.

The closest concurrent work to ours are~\citet{lee2024mechanistic} and ~\citet{jain2023mechanistically}.  
\citet{lee2024mechanistic} investigate the representation and elicitation of toxicity in a GPT-2~\cite{radford2019language} model, and explores via probing how aligning the model using Direct Preference Optimization (DPO)~\cite{rafailov2023direct} mitigates toxicity. Their findings suggest that DPO does not eliminate the model's ability to generate toxic outputs, but rather learns to bypass
the regions that elicit toxicity. While~\citet{lee2024mechanistic} reveal the fragility of model alignment by probing the GPT-2 model at the granularity of per attention head level, our study examines the more advanced Llama family models~\cite{touvron2023llama-2} using per-neuron or per-rank attribution, which is more relevant to real-world applications and allows for a more fine-grained analysis. \citet{jain2023mechanistically} investigate the impact of fine-tuning on LLMs, using methods such as probing and data-agnostic structured pruning. They suggest that fine-tuning model weights might create a `safety wrapper' around core models, rendering the effects of safety alignment easily reversible. In contrast to their approach which operates on the transformer block level, our study examines the models at a more fine-grained neuron level and rank level.

\subsection{Low Rank Compression}
\label{sec:lowrank}
Our work borrows insight from \citet{hsu2022language, schotthofer2022low, zhang2023adaptive, yuan2023asvd, li2023losparse}. We address two similar methods~\citep{yuan2023asvd, hsu2022language} and point out their differences from \ours. \citet{yuan2023asvd} propose ASVD (Activation-aware Singular Value Decomposition) as a low-rank compression technique for language models. Their method performs SVD on $WS$, where $S$ is a diagonal matrix with $S_{ii}$ given by the norm of the activations 
\begin{equation*}
    S_{ii} = \left(\frac{1}{n} \sum_{j}|X_{ij}|\right)^\alpha ~~~\textrm{or} ~~~\left(\max_{j}|X_{ij}|\right)^\alpha.
\end{equation*}
\citet{hsu2022language} propose FWSVD (Fisher-Weighted SVD) as a low-rank compression technique, where the SVD is applied to $\hat{I}W$. Here 
\begin{equation*}
    \hat{I} = \mathrm{diag}(\sqrt{\hat{I}_{W_1}}, \dots, \sqrt{\hat{I}_{W_d}}), ~~~\hat{I}_{W_i} = \sum_{j} \hat{I}_{W_{ij}}.
\end{equation*}
The $\hat{I}_{W_{ij}}$ is defined as the Fisher information of the loss function with respect to the weight entry $W_{ij}$
\begin{equation*}
    \hat{I}_{W_{ij}} = \frac{1}{D} \sum_{i} \Big(\partial_{W_{ij}} \mathcal{L}(x_i; W_{ij})\Big)^2.
\end{equation*}
In contrast, for the proposed method \ours, we perform SVD on $WX_{\mathrm{in}}$, where $X_{\mathrm{in}}$ is the stack of all activations, i.e., $X_{\mathrm{in}} = [X_1, \dots, X_n] \in \mathbb{R}^{d_{\mathrm{in}} \times n}$.
 
\subsection{Pruning}
\label{sec:related_pruning}
Our approach to attribution aligns closely with techniques used in neural network pruning. Our SNIP method~\cite{lee2019snip} resembles more closely with unstructured pruning techniques, which are designed to establish criteria based on weight magnitude, activations, or network gradients for removing individual weights from a network~\citep{han2015deep, molchanov2016pruning, frankle2018lottery,chen2020lottery,sanh2020movement, zhao2020masking, cao2021low, guo2020parameter}. These unstructured pruning methods have been adapted for use in large language models, as seen in Wanda~\citep{sun2023simple}, SparseGPT~\citep{frantar2023sparsegpt}.

Broadly speaking, the low-rank compression techniques are akin to 
structured pruning approaches, with a focus on identifying important structured subnetworks. In computer vision settings, it is common to remove channels or filters~\citep{li2017pruning, molchanov2016pruning,wen2016learning,he2017channel,luo2017thinet, liu2017learning} from convolutional neural networks. Structured pruning of language models involves removing heads, dimensions, or ranks~\citep{michel2019sixteen,wang2020structured,lagunas2021block, xia2022structured,ma2023llmpruner,zhang2023adaptive, zhang2023pruning, chen2023lorashear, xia2023sheared, ashkboos2024slicegpt}.

While pruning is commonly employed for model compression to decrease model sizes, our work adopts similar techniques to identify critical regions responsible for safety.

%% file: appendices/appendix_exp_details.tex
\section{Experimental Details}
\label{sec:app_exp_detail}

\paragraph{Compute configurations} All the experiments are done with four AMD EPYC 7J13 64-core CPUs and a single NVIDIA A100-80G GPU. During the experiments, we utilize vLLM~\cite{kwon2023efficient} for faster decoding. The typical GPU hours for different experiments are listed in \Cref{tab:gpu-hours}. 

\input{tables/gpu_hours}

\input{appendices/appendix_pruning}

\paragraph{Collection of safety and utility dataset} \Cref{tab:data_info} provides more details for the safety and utility datasets we use in our experiments. 
During the experiment, we sample $128$ (prompt, response) pairs in computing the importance score or projection matrix. 
\input{tables/dataset_info}

\paragraph{Repeat times} To mitigate the potential variability introduced by random seeds, we repeat our experiments on \Cref{sec:exp_removal_bottom} three times with different random seeds. In our figure, we plot the mean value $\mu$ for each data point. To represent variability, we shade the area between $[\mu-\sigma, \mu+\sigma]$, where $\sigma$ denotes the standard deviation corresponding to each point.

\paragraph{The probing baseline} We adopt a similar probing setup used in \citet{li2023inference} for identifying safety-critical neurons in this work.  Specifically, we feed the model with all $420$ harmful instructions from AdvBench$_\textrm{attr}$, as well as $420$ harmless instructions randomly sampled from the utility dataset. We collect the activation outputs of every internal attention head for these instructions. This collected data is then split into two sets, with a $5:2$ ratio for the training split and the validation split, respectively. For each attention head, we train a linear classifier on the training split using its activation inputs to distinguish between activations resulting from harmful and harmless instructions. We then evaluate the accuracy of the classifier on the validation split, which indicates the relevance of the attention head in distinguishing between harmful and harmless instructions.

\paragraph{The adversarial suffixes} We run the GCG attack \cite{zou2023universal} for $500$ iterations,
with adversarial sting initiated as ``\texttt{!!!!!!!!!!!!!!!!!!!!}". For optimization, we use a batch size of $512$, top-$k$ as $256$, with a joint optimization over Llama2 family~\cite{touvron2023llama-2} and Vicuna~\cite{vicuna2023} models\footnote{According to \citet{zou2023universal}, incorporating a broader range of models during training enhances the effectiveness of attacks.}, with their system prompts removed, for three independent trails. We then identify the top three suffixes with the highest attack success rates on \AdvBench, and use them in our evaluation. For ethical reasons, we refrain from disclosing these suffixes to prevent potential misuse. %

\paragraph{The adversarial decoding} For evaluating ASR$_\textrm{Adv-Decoding}$, we configure the sampling temperature to $1.0$ when generating responses from \AdvBench $_\textrm{eval}$. For each harmful prompt in \AdvBench, we perform sampling $5$ times. An attack is considered successful if at least one of the sampled responses is deemed harmful.

\paragraph{The details of the zero-shot tasks in evaluating utility} 
\begin{enumerate}
    \item ARC-Challenge:
    \begin{enumerate}
        \item Downstream Task: Science Question Answering.
        \item Description: The ARC-Challenge metric evaluates the performance of models on the ARC-Challenge subset of the AI2 Reasoning Challenge dataset, which consists of grade-school science questions that require complex reasoning and understanding of scientific concepts \footnote{More details are available at \url{https://allenai.org/data/arc}.}.
    \end{enumerate}
    \item HellaSWAG:
    \begin{enumerate}
        \item Downstream Task: Commonsense Reasoning
        \item Description: HellaSWAG is a dataset for evaluating commonsense reasoning in AI systems. It consists of context and multiple-choice endings, where the task is to predict the most plausible ending. The dataset is designed to test a model's ability to reason about everyday scenarios\footnote{More details are available at \url{https://huggingface.co/datasets/Rowan/hellaswag}.}.
    \end{enumerate}
    \item OpenBookQA:
    \begin{enumerate}
        \item Downstream Task: Open-Book Question Answering
        \item Description: OpenBookQA aims to promote research in advanced question-answering, probing a deeper understanding of both the topic (with salient facts summarized as an open book, also provided with the dataset) and the language it is expressed in. In particular, it contains questions that require multi-step reasoning, use of additional common and commonsense knowledge, and rich text comprehension. OpenBookQA is a new kind of question-answering dataset modeled after open book exams for assessing human understanding of a subject\footnote{More details are available at \url{https://huggingface.co/datasets/allenai/openbookqa}.}.
    \end{enumerate}
    \item WiNoGrande:
    \begin{enumerate}
        \item Downstream Task: Commonsense Reasoning
        Description: WiNoGrande is a dataset for evaluating large-scale commonsense reasoning. It is inspired by Winograd Schema Challenge (Levesque, Davis, and Morgenstern 2011), but adjusted to improve the scale and robustness against the dataset-specific bias. Formulated as a fill-in-a-blank task with binary options, the goal is to choose the right option for a given sentence which requires commonsense reasoning\footnote{More details are available at \url{https://huggingface.co/datasets/winogrande}.}.
    \end{enumerate}
    \item BoolQ:
    \begin{enumerate}
        \item Downstream Task: Yes/No Question Answering
        \item Description: BoolQ is a question answering dataset for yes/no questions containing 15942 examples. These questions are naturally occurring ---they are generated in unprompted and unconstrained settings. Each example is a triplet of (question, passage, answer), with the title of the page as optional additional context. The text-pair classification setup is similar to existing natural language inference tasks\footnote{More details are available at \url{https://github.com/google-research-datasets/boolean-questions}.}.
    \end{enumerate}
    \item RTE (Recognizing Textual Entailment):
    \begin{enumerate}
        \item Downstream Task: Textual Entailment
        \item Description: RTE is a task that involves determining whether a given hypothesis can logically be inferred from a given premise. The dataset consists of pairs of sentences, and the task is to classify each pair as either "entailment" (the hypothesis follows from the premise) or "not entailment" (the hypothesis does not follow from the premise) \footnote{More details are available at \url{https://huggingface.co/datasets/nyu-mll/glue\#rte}.}.
    \end{enumerate}

\end{enumerate}

%% file: tables/gpu_hours.tex
\begin{table}[h]
\centering
\begin{tabular}{cccc}
\toprule
{\bf Model Name}                                  & {\bf Attribution unit}              & {\bf Experiment type}                     & {\bf GPU hours} \\
\midrule
\multirow{4}{*}{\llama}      & \multirow{2}{*}{Neuron} & Pruning Top/Least Safety Neurons    & $0.2$       \\
                                            &                         & Pruning with Set Difference         & $0.2$       \\
                                            \cmidrule{2-4}
                                            & \multirow{2}{*}{Rank}   & Removing Top/Least Safety Ranks     & $0.5$       \\
                                            &                         & Removing with Orthogonal Projection & $1.0$       \\
\midrule
\multirow{4}{*}{\llamalarge} & \multirow{2}{*}{Neuron} & Pruning Top/Least Safety Neurons    & $0.5$       \\
                                            &                         & Pruning with Set Difference         & $0.5$       \\
                                            \cmidrule{2-4}
                                            & \multirow{2}{*}{Rank}   & Removing Top/Least Safety Ranks     & $0.8$       \\
                                            &                         & Removing with Orthogonal Projection & $2.0$       \\
\bottomrule
\end{tabular}
\caption{Typical GPU hours of different experiment types.}
\label{tab:gpu-hours}
\end{table}

%% file: appendices/appendix_pruning.tex
\paragraph{Details for pruning}
For all the methods in the paper, we adopt block-wise pruning as \citet{sun2023simple}, where we start from the first Transformer block in Llama. After pruning the 7 linear layers in the current block ($\mathrm{self\_attn.q}$, $\mathrm{self\_attn.k}$, $\mathrm{self\_attn.v}$, $\mathrm{self\_attn.o}$, $\mathrm{mlp.up}$, $\mathrm{mlp.gate}$, $\mathrm{mlp.down}$), we \textit{recompute} the outputs of the current block and continue to the next block.
 
For the neuron-level attribution, we use output-wise pruning following \citet{sun2023simple}, as the authors observed that pruning per output has better performance for language models. Specifically, after we obtain the score matrix $I(W)$, for a specific sparsity ratio $p \%$, we set $p \%$ of the weights to zero \textit{independently for each row} of the matrix $W$.

%% file: tables/dataset_info.tex
\begin{table}[!ht]
\centering
\begin{tabular}{cccc}
\toprule
{\bf Dataset Name} & {\bf Data sources}            & {\bf Number of (prompt, response) pairs}  & {\bf Average Sequence Length}\\
\midrule
Safety-full       & \AdvBench$_\textrm{attr}$ (prompt) & $7,220 $  &         $175.68$                     \\
Safety-short       & \AdvBench$_\textrm{attr}$ (prompt) & $7,220 $   &         $48.78$              \\
Utility      & \Alpacacleaned         & $45,874 $                  &        $118.50$   \\
\bottomrule
\end{tabular}
\caption{The basic information of datasets used for attribution in our experiments.}
\label{tab:data_info}
\end{table}

%% file: appendices/appendix_exp_results.tex
\section{More Experimental Results}

\subsection{More results in \llamalarge}\label{sec:13b_results}

We plot the results of removing the most safety-critical neurons and ranks on \Cref{fig:exp_main_13b} and the results of removing the least safety-critical neurons and ranks on \Cref{fig:ASR_vs_sparsity_13b}.

\begin{figure*}[!ht]
\begin{minipage}[b]{\linewidth}
\centering
\includegraphics[width=0.9\linewidth]{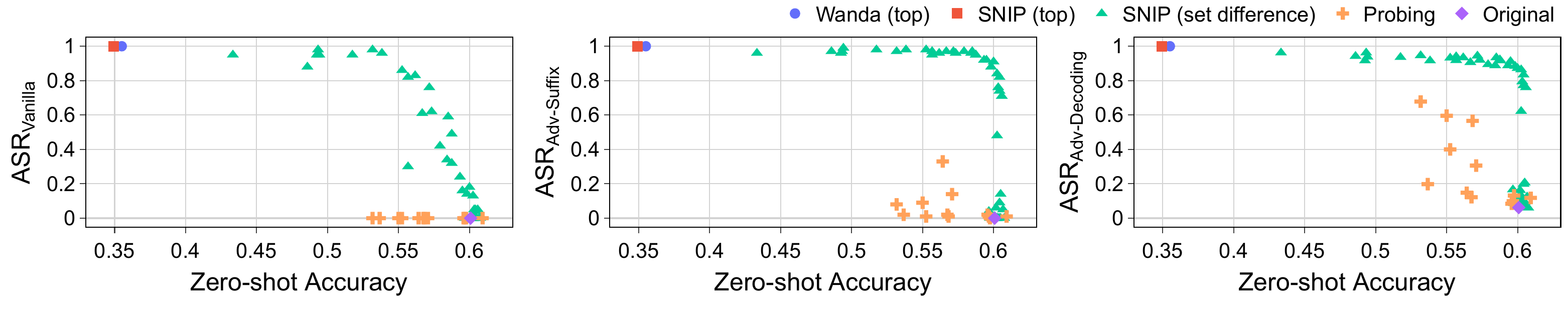}
\subcaption{Removing Safety-Critical \textbf{Neurons}}
\label{fig:neuron_13b}
\end{minipage}
\begin{minipage}[b]{\linewidth}
\centering
\includegraphics[width=0.9\linewidth]{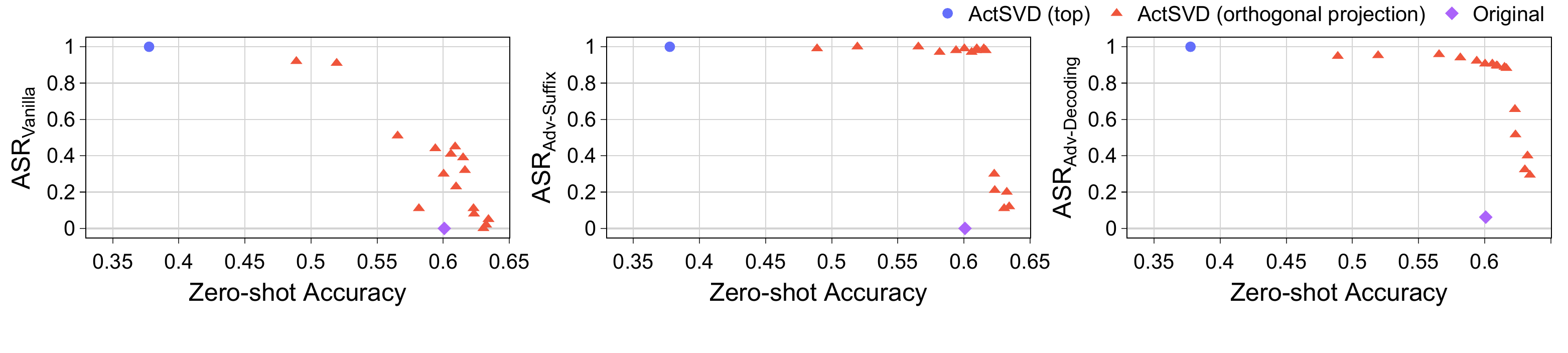}
\vspace{-3mm}
\subcaption{Removing Safety-Critical \textbf{Ranks} 
}
\label{fig:rank_13b}
\end{minipage}
\vspace{-5mm}
\caption{ASR and accuracy after removing safety-critical neurons in \llamalarge identified by \textbf{(a)} different methods in \Cref{sec:exp_setup_neurons} with \textbf{sparsity constraint $\mathbf{<3\%}$}; \textbf{(b)} different methods in \Cref{sec:exp_setup_ranks} with \textbf{{ranks of the weight updates $\leq \mathbf{120}$}} (out of $5120$).
 Among all methods, disentangling safety from utility (set difference for neurons and orthogonal projection for ranks) mostly effectively identify the safety-critical regions, with safety severely compromised while utility retains.
}
\label{fig:exp_main_13b}
\end{figure*}

In accordance with the results on \llama (see \Cref{sec:exp_results_main}), from \Cref{fig:exp_main_13b}, we observe similar results on \llamalarge: 
\begin{itemize}
    \item Removing safety-critical neurons using set difference, or removing safety-critical ranks using orthogonal projection, is effective in destroying the model's safety while preserving utility.
    \item Removing top safety neurons or ranks severely hurts utility.
    \item Set difference with SNIP score consistently outperforms the attention head probing baseline.
\end{itemize}

However, from \Cref{fig:ASR_vs_sparsity_13b}, we observe different curves from the results for \llama. 
\begin{itemize}
    \item The exclusion of neurons using SNIP and \ours deemed least critical for safety slightly enhances robustness against adversarial decoding attacks, i.e., when actual sparsity $>0.4$ \& removed rank $<300$. 
    \item In contrast, removing neurons according to the least Wanda scores hurts the adversarial robustness. 
    \item Different from \llama, we see that the original \llamalarge has zero ASR$_\textrm{Adv-Suffix}$. Removing less than $45\%$ of neurons or less than $750$ ranks that are least critical for safety \textit{maintains} the robustness against adversarial suffixes at a nearly $0$ ASR. 
    One potential reason behinds the phenomenon is that the adversarial suffixes are obtained using 7B models and they cannot transfer to \llamalarge. It may be possible that the trend between the \llamalarge and \llama models becomes more aligned with optimized suffixes.
\end{itemize}

\begin{figure*}[t]
    \begin{minipage}{0.49\linewidth}
    \centering
    \includegraphics[width=\linewidth]{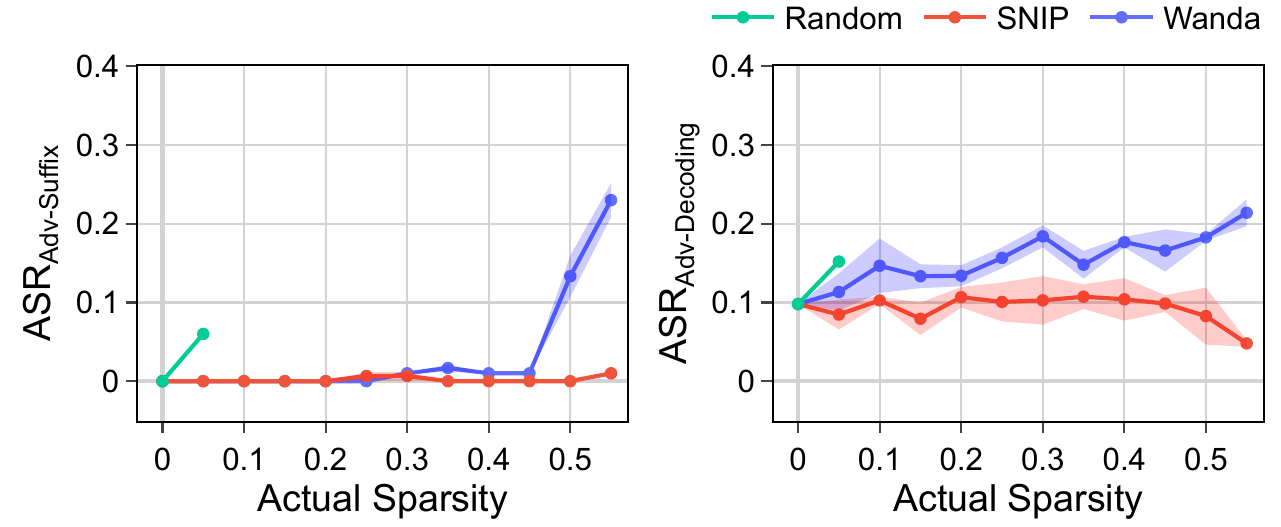}
    \subcaption{Removing the least Safety-Critical \textbf{Neurons}}
    \label{fig:ASR_vs_sparsity_neuron_13b}
    \end{minipage}
    \hfill
    \begin{minipage}{0.49\linewidth}
    \centering
    \includegraphics[width=\linewidth]{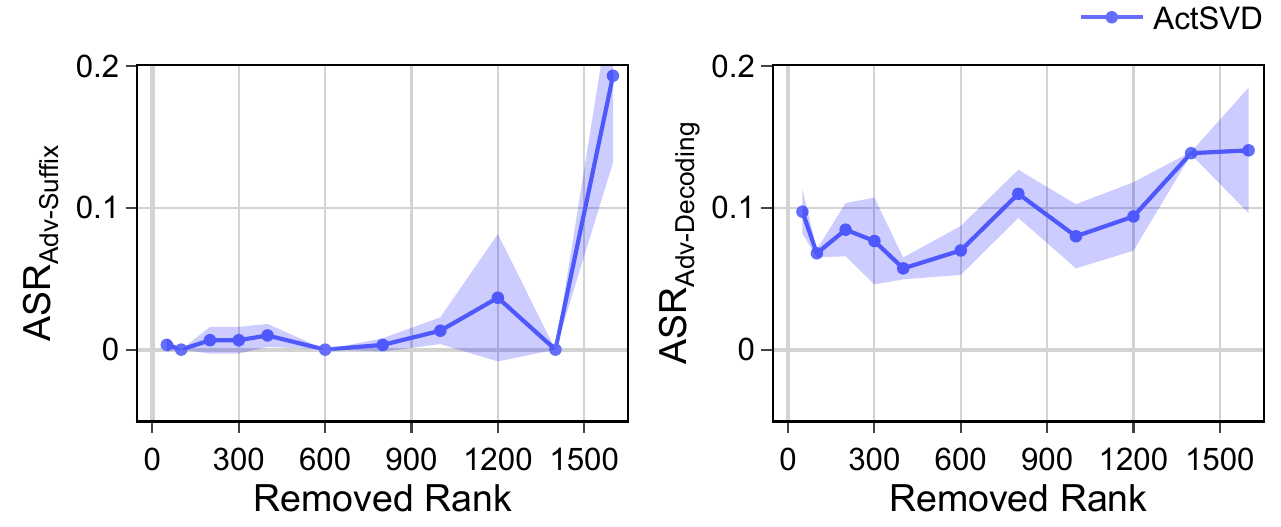}
    \subcaption{Removing the least Safety-Critical \textbf{Ranks}}
    \label{fig:ASR_vs_sparsity_rank_13b}
    \end{minipage}
    \caption{Impact on ASR under adversaries in \llamalarge when: (a) Removing neurons with  \textit{the lowest} safety importance scores and (b) Removing \textit{the least} safety-critical ranks identified by \ours. Both the importance score and safety-critical ranks are calculated based on safety short. All plotted points have reasonable utility with accuracy $>0.5$. 
    The exclusion of neurons (using SNIP) and ranks deemed least critical for safety slightly enhances robustness against adversarial decoding attacks (when actual sparsity $<0.4$ \& removed rank $<300$), and it maintains the robustness at a nearly $0$ ASR against adversarial suffixes.
    }
    \label{fig:ASR_vs_sparsity_13b}
\end{figure*}

\subsection{Ablation Study between safety-full dataset and safety-short dataset}\label{sec:safety_vs_safety_short}

\paragraph{Performance of set difference} As shown in \Cref{fig:ASR_vs_acc_full_vs_short}, the trends in ASR versus accuracy for both the safety-full and safety-short attribution datasets are similar. This observation implies that utilizing judgment-only data is as effective as using the full response for identifying safety-critical neurons.

\begin{figure}[!ht]
    \centering
    \includegraphics[width=0.95\linewidth]{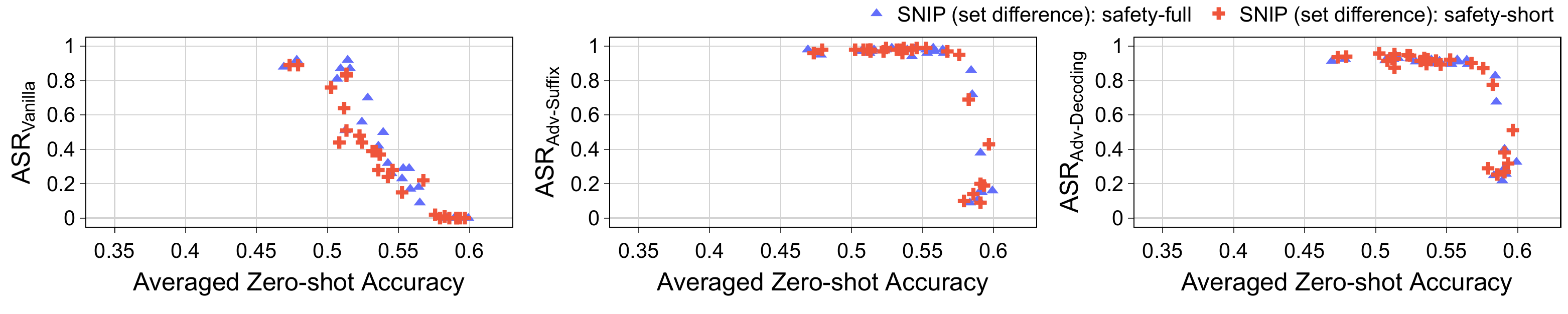}
    \vspace{-5mm}
    \caption{The relationship between ASR and Averaged Zero-shot Accuracy with set difference pruning methods on \llama, using safety and safety-short datasets. The model's performance is very similar in safety and safety-short.}
    \label{fig:ASR_vs_acc_full_vs_short}
\end{figure}

\paragraph{Performance of pruning the least safety-critical region} As shown in \Cref{fig:ASR_vs_sparsity_short_vs_full}, when pruning the least safety-critical region, compared to safety-full dataset, using safety-short dataset exhibits a more significant change in both ASR$_\textrm{Adv-Suffix}$ and ASR$_\textrm{Adv-Decoding}$. We also observe that ASR$_\textrm{Vanilla}$ remains close to zero for actual sparsity levels between 0 and 0.55. Therefore, we only report results for ASR$_\textrm{Adv-Suffix}$ and ASR$_\textrm{Adv-Decoding}$, with safety-short in \Cref{sec:exp_removal_bottom}.

\begin{figure}[!ht]
    \centering
  \includegraphics[width=0.6\textwidth]{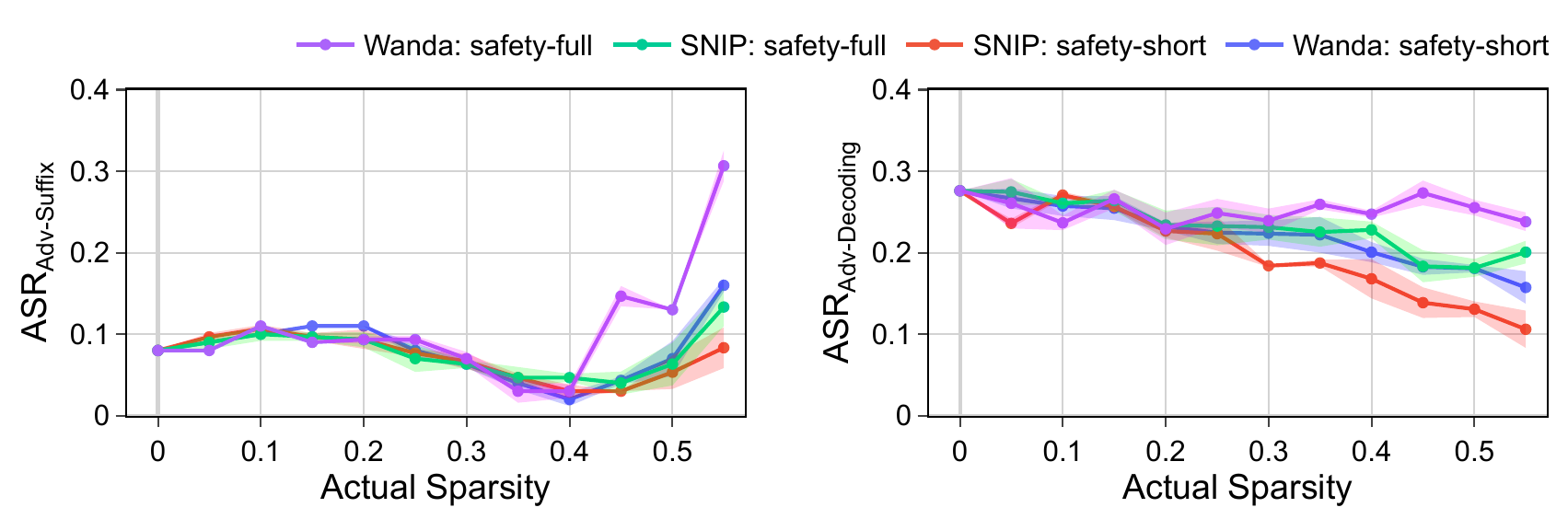} 
  \vspace{-5mm}
  \caption{Comparison between using safety-full and safety-short when pruning the least safety-critical neurons on \llama. Compared to safety-full dataset, using safety-short dataset exhibits a more significant change.}
  \label{fig:ASR_vs_sparsity_short_vs_full}
\end{figure}

\subsection{More Results for Disentanglement Methods}
\label{sec:pq_comb}

\paragraph{More results for $(p, q)$ combinations in set difference.} We conduct a comprehensive study exploring the search space for $(p, q)$ values ranging between $0.1$ and $90$. Complementing \Cref{fig:neuron_7b} and \Cref{fig:neuron_13b}, \Cref{tab:pq_comb} presents the top $(p, q)$ combinations along with the utility and safety measures for the resulting models on \llama and \llamalarge. In scenarios where the actual sparsity is less than $1\%$, the model maintains a low ASR$_\textrm{Vanilla}$, typically under $0.3$. However, its ASR$_\textrm{Adv-Suffix}$ and ASR$_\textrm{Adv-Decoding}$ nearly reach to $1$. In contrast, when the actual sparsity $\in (1\%, 3\%)$., the model approaches a value close to $1$ for all three ASR variants. Notably, across all cases outlined in \Cref{tab:pq_comb}, the model consistently maintains utility, with an average accuracy greater than $0.5$. These findings indicate that the optimal range for the $(p, q)$ parameters lies between $3$ and $9$, especially when the values of $p$ and $q$ are similar.

\input{tables/pq_comb_results}
\label{sec:app_exp_results}

\paragraph{More results for $(r^u, r^s)$ combinations in orthogonal projection.}

Note that the ranks of the weight matrices of the linear layers are $R=4096$ for \llama and $R=5120$ for \llamalarge. 
We perform a grid search for the parameters $r^u$ and $r^s$, spanning a range from $50$ to $4000$ for \llama and from $1200$ to $5000$ for \llamalarge. As an extension to \Cref{fig:rank_7b} and 
\Cref{fig:rank_13b}, \Cref{tab:ru_rs_comb} presents the top five combinations of $r^u$ and $r^s$ along with the utility and safety metrics for the models tested on \llama and \llamalarge model. The results indicate that setting $r^s$ close to $R$, especially when $r^u$ closes to $r^s$, proves to be particularly effective.

\input{tables/ru_rs_comb_results}

\subsection{Probing Accuracy Distributions}
\label{sec:app_probing}

We also analyze the accuracy of linear probers trained on all $1024$ attention heads from \llama (\Cref{fig:probe-sub1}) and $1600$ attention heads from \llamalarge (\Cref{fig:probe-sub2}). The results show that around half of the attention heads achieve very high probing accuracy (i.e., $>0.95$) in distinguishing between harmful and harmless instructions. Notably, even the attention heads with the lowest probing accuracy show significant effectiveness -- $0.78$ for \llama and $0.74$ for \llamalarge. Additionally, transformer blocks located in the middle typically demonstrate higher probing accuracy compared to those at the beginning or end.

\begin{figure}[!htbp]
    \centering
    \begin{minipage}[b]{0.48\textwidth}
        \includegraphics[width=\textwidth]{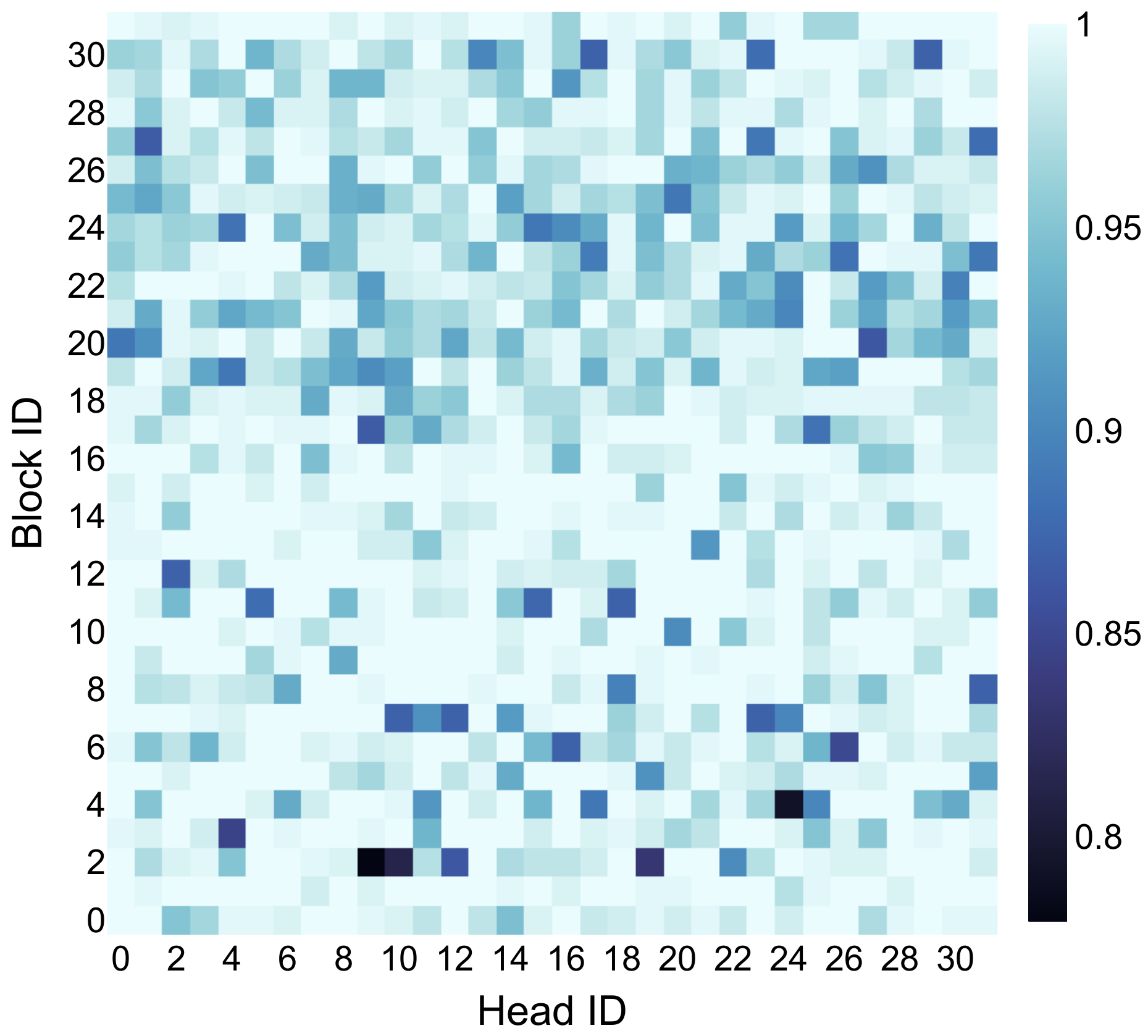}
        \subcaption{\llama}
        \label{fig:probe-sub1}
    \end{minipage}
    \begin{minipage}[b]{0.48\textwidth}
        \includegraphics[width=\textwidth]{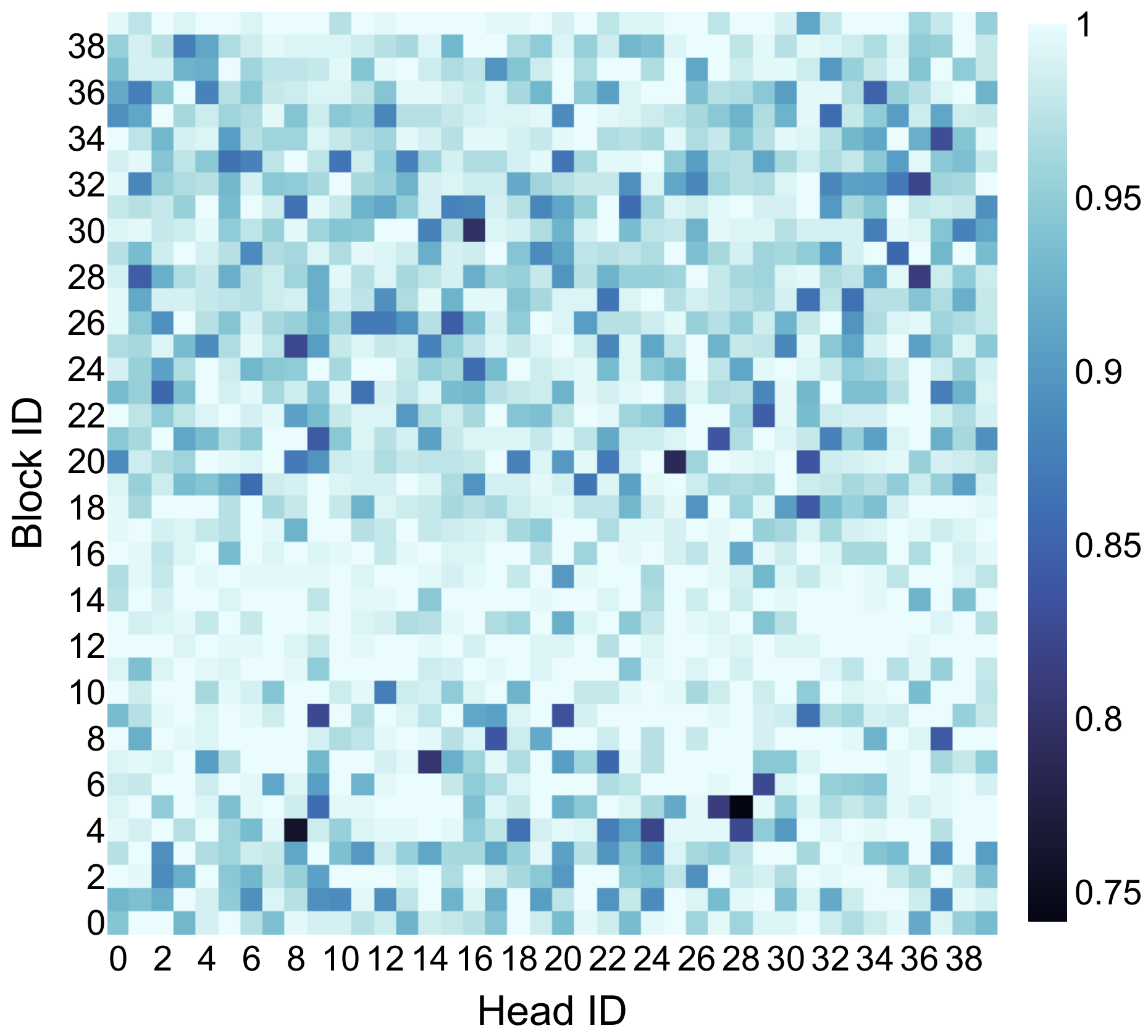}
        \subcaption{\llamalarge}
        \label{fig:probe-sub2}
    \end{minipage}
    \vspace{-3mm}
    \caption{The probing accuracy distribution across different attention heads inside \llama and \llamalarge. Around half of the attention heads achieve very high probing accuracy (i.e., $>0.95$) in distinguishing between harmful and harmless instructions.}
    \label{fig:test}
\end{figure}

This pattern of high accuracy suggests that making safety judgments at the level of individual attention heads is relatively straightforward due to their effective representational capacity. Therefore, our study's focus on finer granularities, such as neurons or ranks, is essential for the precise localization of safety-critical regions.

%% file: tables/pq_comb_results.tex
\begin{table}[!ht]
\begin{minipage}[b]{\linewidth}
\centering
\small
\begin{tabular}{ccccccc}
\toprule
$p$ & $q$ & \textbf{Actual Sparsity} & \textbf{ASR$_\textrm{Vanilla}$} & \textbf{ASR$_\textrm{Adv-Suffix}$} &  \textbf{ASR$_\textrm{Adv-Decoding}$} & \textbf{Averaged Accuracy} \\
\midrule
\multicolumn{7}{c}{\bf Actual Sparsity  $ < 1\%$} \\
\midrule
$1$ & $1$ & $0.63$\%          & $0.10$   & $0.97$ & $0.86$   & $0.54$        \\
$2$ & $1$ & $0.46$\%          & $0.06$   & $0.94$ & $0.84$   & $0.55$       \\
$4$ & $2$ & $0.78$\%          & $0.09$   & $0.97$ & $0.88$   & $0.56$        \\
$7$ & $3$ & $0.86$\%          & $0.06$   & $0.97$ & $0.89$   & $0.58$       \\
$3$ & $2$ & $0.94$\%          & $0.21$   & $0.97$ & $0.89$   & $0.55$        \\
\midrule
\multicolumn{7}{c}{\bf Actual Sparsity $\in (1\%, 3\%)$} \\
\midrule
$4$ & $4$ & $2.03$\%          & $0.81$   & $0.97$ & $0.92$   & $0.51$        \\
$5$ & $5$ & $2.41$\%          & $0.92$   & $0.97$ & $0.95$   & $0.51$        \\
$6$ & $5$ & $2.10$\%          & $0.70$   & $0.99$ & $0.91$   & $0.53$        \\
$6$ & $6$ & $2.75$\%          & $0.87$   & $0.97$ & $0.93$   & $0.51 $       \\
$9$ & $8$ & $2.99$\%          & $0.87$   & $0.98$ & $0.93$   & $0.52$       \\
\bottomrule
\end{tabular}
\subcaption{\llama}
\label{tab:neuron_sub_7b}
\end{minipage}
\begin{minipage}[b]{\linewidth}
\small
\centering
\begin{tabular}{ccccccc}
\toprule
$p$ & $q$ & \textbf{Actual Sparsity} & \textbf{ASR$_\textrm{Vanilla}$} & \textbf{ASR$_\textrm{Adv-Suffix}$} &  \textbf{ASR$_\textrm{Adv-Decoding}$} & \textbf{Averaged Accuracy} \\
\midrule
\multicolumn{7}{c}{\bf Actual Sparsity  $ < 1\%$}
\\
\midrule
$1$ & $1$ & $0.66$\%          & $0.30$   & $0.82$ & $0.83$   & $0.60$        \\
$5$ & $2$ & $0.72$\%          & $0.18$   & $0.91$ & $0.87$   & $0.60$        \\
$7$ & $2$ & $0.53$\%          & $0.13$   & $0.84$ & $0.86$   & $0.60$        \\
$7$ & $3$ & $0.93$\%          & $0.16$   & $0.92$ & $0.91$   & $0.60$       \\
$8$ & $3$ & $0.82$\%          & $0.14$   & $0.88$ & $0.88$   & $0.60$       \\
\midrule
\multicolumn{7}{c}{\bf Actual Sparsity $\in (1\%, 3\%)$} \\
\midrule
$3$ & $3$ & $1.72$\%          & $0.96$   & $0.98$ & $0.92$   & $0.54$        \\
$4$ & $3$ & $1.46$\%          & $0.86$   & $0.98$ & $0.93$   & $0.55$        \\
$4$ & $4$ & $2.15$\%          & $0.98$   & $0.97$ & $0.95$   & $0.53$        \\
$6$ & $6$ & $2.91$\%          & $0.95$   & $0.98$ & $0.94$   & $0.52$       \\
$9$ & $7$ & $2.59$\%          & $0.83$   & $0.96$ & $0.93$   & $0.56$       \\
\bottomrule
\end{tabular}
\subcaption{\llamalarge}
\label{tab:neuron_sub_13b}
\end{minipage}
\caption{Performance of \llama (a) and \llamalarge (b) with safety-critical neurons removed, identified through set difference between top-$q\%$ safety and top-$p\%$ utility neurons, across various $(p, q)$ combinations. The ideal range for $(p, q)$ is $[3, 9]$, especially when $p$ and $q$ are closely matched.}
\label{tab:pq_comb}

\end{table}

%% file: tables/ru_rs_comb_results.tex
\begin{table}[!ht]
\begin{minipage}[b]{\linewidth}
\centering
\small
\begin{tabular}{ccccccc}
\toprule
$r^u$ & $r^s$ & $\min(r^u, R-r^s)$ & \textbf{ASR$_\textrm{Vanilla}$} & \textbf{ASR$_\textrm{Adv-Suffix}$} &  \textbf{ASR$_\textrm{Adv-Decoding}$} & \textbf{Averaged Accuracy} \\
\midrule
$3450$ & $4000$ & $96$          & $0.67$   & $1.00$ & $0.88$   & $0.59$        \\
$3550$ & $4000$ & $96$          & $0.68$   & $0.99$ & $0.90$   & $0.59$        \\
$3950$ & $4090$ & $6$         & $0.71$   & $0.97$ & $0.91$   & $0.59 $       \\
$4000$ & $4090$ & $6$         & $0.71$   & $0.97$ & $0.92$   & $0.58$        \\
$4080$ & $4090$ & $6$          & $0.65$   & $0.98$ & $0.94$   & $0.57$       \\

\bottomrule
\end{tabular}
\subcaption{\llama ($R=4096$)}
\label{tab:rank_sub_7b}
\end{minipage}
\begin{minipage}[b]{\linewidth}
\centering
\small
\begin{tabular}{ccccccc}
\toprule
$r^u$ & $r^s$ & $\min(r^u, R-r^s)$ & \textbf{ASR$_\textrm{Vanilla}$} & \textbf{ASR$_\textrm{Adv-Suffix}$} &  \textbf{ASR$_\textrm{Adv-Decoding}$} & \textbf{Averaged Accuracy} \\
\midrule
$3450$ & $5000$ & $120$          & $0.32$   & $0.98$ & $0.88$   & $0.62$        \\
$3600$ & $5000$ & $120$         & $0.41$   & $0.97$ & $0.91$   & $0.61 $       \\
$3900$ & $5000$ & $120$         & $0.30$   & $0.99$ & $0.91$   & $0.60$        \\
$3750$ & $5000$ & $120$          & $0.39$   & $0.99$ & $0.89$   & $0.62$       \\
$4400$ & $5000$ & $120$          & $0.91$   & $1.00$ & $0.95$   & $0.52$       \\

\bottomrule
\end{tabular}
\subcaption{\llamalarge ($R=5120$)}
\label{tab:rank_sub_13b}
\end{minipage}

\caption{Performance of  \llama (a) and \llamalarge (b) model with safety-critical ranks removed by doing orthogonal projection between utility projection matrix $\Pi^u$ and safety projection matrix $\Pi^s$, across various $(r^u, r^s)$ combinations. Setting $r^s$ close to $R$, especially when $r^u$ closes to $r^s$, proves to be particularly effective.}
\label{tab:ru_rs_comb}

\end{table}

%% file: appendices/appendix_proof.tex
\section{Proof of the Optimality of \ours}
\label{sec:proof}
\newtheorem{lemma}{Lemma}

\begin{lemma}
Let $X_{\mathrm{in}} \in \mathbb{R}^{d_\mathrm{in} \times n}$. Let $\widehat{W}$ be the solution to the following rank-constrained approximation problem.
\begin{equation}
    \widehat{W} = \argmin_{\rank {\widehat{W}} \leq r} \| W X_{\mathrm{in}} - \widehat{W} X_{\mathrm{in}} \|_F^2. \label{eq:obj}
\end{equation}
    
Let $U S V^\top$ be the rank-$r$ SVD on $W  X_{\mathrm{in}} \in \mathbb{R}^ {d_\mathrm{out} \times n}$:
\[
    U S V^\top  \approx WX_{\mathrm{in}}, 
\]
where $U \in \mathbb{R}^{d_\mathrm{out}\times r}$ is the orthogonal matrix corresponding to the top $r$ left singular vectors. The minimizer to \Cref{eq:obj} is given by 
\[
    \widehat{W} = UU^\top W.
\]
\end{lemma}

\begin{proof}
Denote $Z = WX_{\mathrm{in}} $. By Eckart–Young–Mirsky theorem~\citep{eckart1936approximation}, we know that the SVD  $\widehat{Z} = U S V^\top  \approx Z$ is the best rank-$r$ approximation to $Z$, where $U \in \mathbb{R}^{d_\mathrm{out} \times r}$, $S = \mathrm{diag}(S_1,\dots, S_r)$, $V \in \mathbb{R}^{n \times r}$.  Furthermore, we have 
\[
    \widehat{Z} = UU^\top Z.
\]
Plugging in $Z = WX_{\mathrm{in}} $, we have 
\[
  \widehat{Z} =   UU^\top  W X_{\mathrm{in}}.
\]
Recall that we set $\widehat{W} = UU^\top  W$. We see that 
\[
    \| \widehat{Z} - Z \|_F^2 \text{ is minimized } \Rightarrow \| \widehat{W}X_{\mathrm{in}} - WX_{\mathrm{in}} \|_F^2 \text{ is minimized. } 
\]
Furthermore, as $UU^\top$ is a rank-$r$ projection matrix, we have $\rank(\widehat{W}) \leq r$. Therefore, $\widehat{W}$ is the optimal solution to the rank-constrained minimization problem~(\Cref{eq:obj}). The same reasoning is used in \citet{hsu2022language}.
\end{proof}